\pdfoutput=1

\documentclass[english,11pt]{article}
\include{macros}

\begin{document}

\title{Adaptive Data Fusion for Multi-task Non-smooth Optimization}\blfootnote{Author names are sorted alphabetically.}

\author{Henry Lam\thanks{Department of IEOR, Columbia University. Email: \texttt{henry.lam@columbia.edu}.}
	\and Kaizheng Wang\thanks{Department of IEOR and Data Science Institute, Columbia University. Email: \texttt{kaizheng.wang@columbia.edu}.}
	\and Yuhang Wu\thanks{Department of Mathematics, Columbia University. Email: \texttt{yuhang.wu@columbia.edu}.}
	\and Yichen Zhang\thanks{Krannert School of Management, Purdue University. Email: \texttt{zhang@purdue.edu}.}
}

\date{October 2022}

\maketitle

\begin{abstract}
We study the problem of multi-task non-smooth optimization that arises ubiquitously in statistical learning, decision-making and risk management. We develop a data fusion approach that adaptively leverages commonalities among a large number of objectives to improve sample efficiency while tackling their unknown heterogeneities. We provide sharp statistical guarantees for our approach. Numerical experiments on both synthetic and real data demonstrate significant advantages of our approach over benchmarks.
\end{abstract}


\section{INTRODUCTION}


In most machine-learning contexts, algorithm developers and theorists are concerned with solving a single task or optimizing a single metric at a time. Nonetheless, even in the big data era, the datasets are expensive and oftentimes collected for a large number of tasks, and models based on a single task likely hit the performance ceiling due to the limited sample size without fully exploiting the dataset featuring multiple tasks.
For instance, in inventory management, the hype cycle of technology is getting shortened. It is increasingly critical for retailers to recognize the consumption patterns of customers as early as possible, so as to minimize the cost caused by backordering and holding. Since the selling data is limited at the early stage of the operations, decision making can generally be challenging. Nevertheless, a retailer usually sells multiple products in a store or manages multiple stores that sell the same product. They naturally define a group of related tasks.
The retailer may effectively pool the datasets together to obtain a better estimation and decision. 
By sharing representations between related tasks, multi-task learning conceptually helps the model generalize better on individual tasks \citep{Car97}.

That being said, the relatedness of the tasks is implicit and hard to quantify. An oversighted or misspecified relationship among tasks could adversely hurt the performance of data pooling across multiple tasks. When the tasks are highly distinct, naive multi-task learning procedures could underperform single-task learning (STL) ones which tackle each task separately.


To utilize the commonalities of datasets related to multiple tasks, statistical models can help us develop a family of reliable approaches with theoretical quantification of multi-task performance while adapting to the unknown task relatedness. To this end, we propose a data fusion approach to tackle unknown heterogeneities among the tasks and improve the data utilization for each individual task. Particularly, consider $m$ tasks of empirical risk minimization where the $j$-th task is to minimize the empirical risk $f_j(\btheta) = \frac{1}{n_j} \sum_{i=1}^{n_j}\ell(\btheta, \bxi_{ji})$ over the parameter of interest $\btheta$ from data $\{ \bxi_{ji} \}_{i=1}^{n_j} \sim\cP_j$. We propose to minimize an augmented objective
\begin{equation*}
	\bF_{\bw, \blambda} (\bTheta, \bbeta) = \sum_{j=1}^{m} w_j
\left[ f_j(\btheta_j) + \lambda_j \norm{\btheta_j - \bbeta}_2 \right]
\end{equation*}
jointly over task-specific estimators $\bTheta = (\btheta_1, \cdot\cdot\cdot, \btheta_m) \in \RR^{d \times m}$ and a multi-task center $\bbeta \in \RR^d$ to be learned. Weighting hyperparameters $w_j$ are specified to reflect the importance of the information regarding each individual task (e.g. $w_j = n_j$). 
The regularization term $\lambda_j \norm{\btheta_j - \bbeta}_2 $ drives the estimator of each individual task $\btheta_j$ towards a common center $\bbeta$, with strength parameterized by $(\lambda_1,\cdots, \lambda_m)$. 
It is straightforward to see that our method interpolates between minimizing the risk of the individual task as $\lambda_j$ approaches zero, and a robust pooling of the individual minimizers as $\lambda_j$ increases.

Our key contribution is the analysis of the aforementioned procedure in a wide range of problems where the losses $\{ f_j \}_{j=1}^m$ are convex but allowed to be nonsmooth.
We prove that the proposed estimator automatically adapts to the unknown similarity among the tasks. 
In our motivating example, the cost function is naturally a piecewise linear nonsmooth convex function which is closely related to quantile regression.
Other examples include linear max-margin classifiers as well as threshold regression models. Among these models, since the objective function is not differentiable at many places, technical challenges arise in the uniform concentration results and convergence rates as the subgradient is now a set-valued mapping and not continuous. Nonetheless, with statistical modeling, a theoretical analysis under such scenarios becomes possible. 

In addition to the theoretical guarantees, we experiment with the numerical procedures on both synthetic data and a real-world dataset of the newsvendor problem in Section \ref{sec:exp}. The experiment reveals a steady and reliable benefit of the performance of the proposed method over benchmark ones, with significant improvement over STL where the data are scarce, and over blindly pooling the data together. This proposed method offers a reliable procedure for practitioners to leverage the possible relatedness between tasks in inventory decision-making, financial risk management, and many other applications.

\subsection{Related Work}

Multi-task learning based on parameter augmentation, such as the introduction of the common center $\bbeta$ in our method, has achieved great empirical success \citep{EPo04,JSR10,CZY11}. 
Our estimator originates from the framework of Adaptive and Robust Multi-task Learning (ARMUL) proposed by \cite{DW22}, while we relaxed the smoothness and strong convexity condition on the empirical risks $f_j$, such that we can extend the analysis to many real-world applications from statistical learning to inventory decision-making, to financial risk management. 
The motivating inventory management example, often known as the data-driven newsvendor problem, can be expressed as a quantile regression problem with the quantile level determined by a ratio of per unit holding cost versus the backordering one \citep{LRS07,LPU15,BRu19}. The objective function, also known as the ``check function'', is convex but not differentiable.
These applications coincide with the classical quantile regression in statistics and econometrics literature, dated back to \cite{KBa78}, which estimates the conditional quantile of the response variable across values of predictive covariates. Besides the aforementioned newsvendor problems in inventory management, quantile regression finds a wide range of applications in survival data analysis \citep{KGe01,WWa14}, financial risk management \citep{EMa99,RUr00} and many other fields. 
We refer the reader to \cite{Koe05,Koe17} for an extensive overview of quantile regression.

Immense applications rising in various fields call for the need of generalization to the nonsmooth objectives in multi-task learning. Despite its practical importance, MTL for nonsmooth objectives remains largely unexplored in statistical learning, with several exceptions including \cite{FXZ16,CHY21,KAJ22}. In contrast to our framework, they studied quantile regression with multivariate response variables in linear models and neural networks, respectively, and treated each variable as a different task sharing the same observed covariates. Since their models are based on a shared covariate for multiple tasks, they typically imposed factor structure and augmented the objective with a rank-based regularization on the $\bTheta$ matrix. On the contrary, our framework features different covariates in each task. As such, we regularize the objective function with a penalty driving towards a robust central of all tasks and utilize the information to jointly optimize over the individual estimators and the intrinsic central. 
Our analysis also complements and generalizes limited existing literature on nonsmooth quantile regression in large-scale or distributed datasets \citep{VCC19,CLZ21} which considered only the quantile regression under homogeneous tasks. 
Several other existing works considered a similar framework as ours in an empirical Bayesian argument. \cite{GKa22} developed a data pooling procedure for data-driven newsvendor problems that shrinks the empirical distribution of each individual task towards a weighted global empirical distribution according to an anchor distribution. The data distributions there have finite supports.    
\cite{MBR15} focused on the Gaussian setting, studied the predictive risk instead of estimation error, and proposed a shrinkage estimator  towards a data-driven location simultaneously optimized.

\subsection{Notation}
The constants $c_1, c_2, C, C_1, C_2, \cdot\cdot\cdot$ may differ from line to line. We use $[n]$ as
a shorthand for $\{1, 2, \cdot\cdot\cdot, n\}$ and $\left|\cdot\right|$ to denote the absolute value
of a real number or cardinality of a set. $\norm{\bA} = \sup_{\norm{\bx}_2 = 1} \norm{\bA \bx}_2$
denotes the spectral norm. Let $\ZZ_+$ be the set of positive integers and $\RR_+ = \{ x \in \RR:~ x \geq 0 \}$. Define $x^+ = \max\{x, 0\}$ and $x^- = \max\{-x, 0\}$ for
$x \in \RR$. Define $B(\bx, r) = \{\by \in \RR^d: \norm{\by - \bx}_2
\le r\}$ for $\bx \in \RR^d$ and $r \ge 0$.

\section{PROBLEM FORMULATION} \label{sec:prob}

Let $m \in \ZZ_+$ be the number of tasks and $\cX$ be the sample space, and let $\ell: \RR^d
\times \cX \rightarrow \RR$ be a (non-smooth) loss function. For every $j \in [m]$, let $\cP_j$ be a
probability distribution over $\cX$ and $\cD_j = \{\bxi_{ji}\}_{i=1}^{n_j}$ be $n_j$ independent samples drawn
from $\cP_j$. The $j$-th task is to estimate the population loss minimizer 
$$
\btheta^*_j \in \mathop{\mathrm{argmin}}_{\btheta \in \RR^d} \mathbb{E}_{\bxi \sim \cP_j}
\ell(\btheta, \bxi)
$$
from the data. Denote by $\bTheta^* = (\btheta^*_1, \cdot\cdot\cdot, \btheta^*_m) \in \RR^{d \times m}$ the parameter matrix. 

Define the empirical loss function of the $j$-th task as 
$$
f_j(\btheta) =
\frac{1}{n_j}\sum_{i=1}^{n_j} \ell(\btheta, \bxi_{ji}).
$$
Two straightforward strategies are
single-task learning (STL) and data pooling (DP). The former corresponds to solving the individual tasks separately, i.e., 
$$
\widehat{\btheta}_j = \mathop{\mathrm{argmin}}_{\btheta \in \RR^d}
f_j(\btheta),\;\;\;\;\;\; \forall j \in [m].
$$
The latter corresponds to merging all datasets to train a single model, i.e., 
$$
\widehat{\btheta}_1 = \cdot\cdot\cdot =
\widehat{\btheta}_m = \mathop{\mathrm{argmin}}_{\btheta \in \RR^d} \sum_{j=1}^{m} n_j f_j(\btheta).
$$ 
These two strategies have intrinsic shortcomings: STL does not take full advantage of the data
available, while DP has a high risk of model misspecification. To resolve this issue, define 
\begin{equation}
	\bF_{\bw, \blambda} (\bTheta, \bbeta) = \sum_{j=1}^{m} w_j
\left[ f_j(\btheta_j) + \lambda_j \norm{\btheta_j - \bbeta}_2 \right],
\end{equation}
where $\bTheta = (\btheta_1, \cdot\cdot\cdot, \btheta_m) \in \RR^{d \times m}$, $\bbeta \in \RR^d$, $\bw = (w_1,
\cdot\cdot\cdot, w_m) \in \RR^m_+$ are weight parameters (e.g. $w_j = n_j$), and $\blambda = (\lambda_1,
\cdot\cdot\cdot, \lambda_m) \in \RR^m_+$ are penalty parameters. We propose to solve an augmented
program  
\begin{equation}\label{pro:vanilla}
	 (\widehat{\bTheta},\, \widehat{\bbeta}  ) \in
	\mathop{\mathrm{argmin}}_{\bTheta \in \RR^{d \times m},\, \bbeta \in \RR^d} \bF_{\bw, \blambda}(\bTheta, \bbeta) 
\end{equation}
where each task receives its own
estimate $\widehat{\btheta}_j$ while the penalty terms shrink $\widehat{\btheta}_j$'s toward
$\widehat{\bbeta}$ to promote similarity among the estimated models. This is a convex program so long as $\{ f_j \}_{j=1}^m$ are all convex. If we choose $\blambda = \zero$, we return to
the STL setting; if we choose sufficiently large $\blambda$, the cusp of the $\ell_2$ penalty at
zero enforces strict equality $\widehat{\btheta}_j = \widehat{\bbeta}$ for all $j \in [m]$,
effectively pooling all the data. Therefore, it is desirable to choose a suitable $\blambda$ such that we can
automatically adapt to whichever situation proves more suitable.

Note that (\ref{pro:vanilla}) belongs to the framework of Adaptive and Robust Multi-task Learning (ARMUL) proposed by \cite{DW22}.
However, theoretical gaurantees of ARMUL require $\{f_j\}_{j=1}^m$ to be smooth and locally strongly
convex near the minimizers. As such, scenarios where ARMUL is powerful
includes, for example, multi-task linear regression and multi-task logistic regression. 
In contrast, our theoretical results relax the smoothness condition and extend to scenarios where
$\{f_j\}_{j=1}^m$ are non-smooth, which are ubiquitous in statistical learning and operations research. 

\section{EXAMPLES}

Here we introduce three motivating examples in statistical learning and operations management.


\subsection{Newsvendor Problem}\label{subsec:newsvendor}

Suppose a retailer sells a perishable good that needs to be
prepared/stocked/ordered in advance. Let $D$ be a random variable representing the
market demand. The retailer needs to decide a quantity $q$ of goods to prepare (e.g. raw materials to buy, food to defrost) ahead of time in
order to minimize the expected cost a combination of the backorder/underage and holding/overage costs as follows,
$$
\mathbb{E}_D \left[  b (D - q)^+ + h(D - q)^- \right],
$$
where $b$ and $h$ are backorder and holding costs per unit, respectively.


In practice, the distribution of the demand $D$ is not known beforehand. Instead, the information available is a set of independent random samples $\{D_i\}_{i=1}^n$ drawn from that. We can estimate
$q(\tau)$ by $\hat{q} \in \mathop{\mathrm{argmin}}_{q \in \RR} f(\theta)$,
where the objective function
$$
f(q) = \frac{1}{n} \sum_{i=1}^{n} \left[  b (D_i - q)^+ + h(D_i - q)^- \right]
$$
is non-smooth. If we define the \emph{check loss} $\rho_\tau(z) = (1 - \tau) z^- + \tau z^+ $, then $f(q)$ is proportional to $ \frac{1}{n} \sum_{i=1}^{n} \rho_{\tau} ( D_i - q )$ with $\tau = b / (b + h)$. The solution $\hat{q}$ is the $\tau$-th sample quantile of the data $\{ D_i \}_{i=1}^n$.

The above classical newsvendor problems assumed that the holding cost and the backordering
cost grow linearly with regard to quantity surplus and deficit, respectively. We can relax this
assumption to cases where the two costs are replaced with $B( (D - q)^+ )$ and $H( (D - q)^- )$ with general functions $B,~ H: \RR_+ \rightarrow \RR_+$ that are convex, non-decreasing, and satisfy $B(0) = H( 0)$. 
Given the data $\left\{   D_i  \right\}_{i=1}^{n}$, it is natural to estimate the best linear decision rule by minimizing the loss function
$$
f(q) = \frac{1}{n} \sum_{i=1}^{n} \big[ B\big( (D_i - q)^+ \big) + H\big( (D_i - q)^- \big) \big].
$$

In modern newsvendor problems, the data for a specific product at one store can be quite limited. Fortunately, multiple products in the same store or multiple stores in a nearby region have similar sales patterns. A joint analysis of the datasets by multi-task learning facilitates decision making.



%

\subsection{Quantile Regression}

Denote by $F_{Y|\bX}$ the conditional CDF of a response $Y \in \RR$ given covariates $\bX \in \RR^d$.
Define the $\tau$-th conditional quantile of $Y $ given $\bX \in \RR^d$ as $Q_{Y|\bX}(\tau) = \inf \left\{y:
F_{Y|\bX}(y) \ge \tau\right\}$. Assume that $Q_{Y|\bX} (\tau) = \bX^\top \btheta^*$ holds for some
$\btheta^* \in \RR^d$. Given $n$ i.i.d. samples $\left\{ ( \bx_i, y_i ) \right\}_{i=1}^{n}$ from some
joint distribution $\cP$, we can estimate $\btheta^*$ by $\widehat{\btheta} \in
\mathop{\mathrm{argmin}}_{\btheta \in \RR^d} f(\btheta)$, where the objective
$$
f(\btheta) = \frac{1}{n}
\sum_{i=1}^{n} \rho_\tau(y_i - \bx_i ^\top \btheta)
$$
is non-smooth. This is the quantile regression in statistics \citep{KBa78} which targets the conditional quantile of the response. In contrast, least squares regression aims to estimate the conditional mean. When we collect data from multiple populations (e.g. different geographical locations), multi-task learning helps utilize their commonality while tackling the heterogeneity.




\subsection{Support Vector Machine}

Consider a binary classification problem where one wants to predict the label $Y \in \{ \pm 1 \}$ from covariates $\bX \in \RR^d$. A popular method for training linear classifiers of the form $\bX \mapsto \sgn( \bX^{\top} \btheta )$ is the support vector machine (SVM) \citep{CVa95}.  Given the data $\{ (\bx_{i} , y_{i}) \}_{i=1}^n$, the (soft-margin) SVM amounts to minimizing the empirical loss function below:
\[
f(\btheta) = \frac{1}{n} \sum_{i=1}^n (1 - y_i \bx_{i}^{\top} \btheta )^+ +  \mu \| \btheta \|_2^2,
\]
where $\mu \geq 0 $ is a penalty parameter. Here $f$ is non-smooth. SVM has demonstrated superior performance in binary classification problems. In multi-class and multi-label settings, multi-task SVM is a popular approach where each task is to distinguish a pair of classes.

\section{THEORETICAL ANALYSIS}\label{sec:theory}

In this section, we provide a non-asymptotic analysis of (\ref{pro:vanilla}). Of particular interest
to us is to generalize the results of \cite{DW22} to non-smooth empirical loss functions. While the empirical loss functions
$\{f_j\}_{j=1}^m$ could be non-smooth, in many cases their population versions (expectations), $\{F_j\}_{j=1}^m$,
have desirable properties such as first-order smoothness and strong convexity. Intuitively, $f_j$
and $F_j$ are ``close", and we can leverage this closeness to bound estimation
errors. In general, this $F_j$ can be any function that is close to $f_j$ and enjoys
the aforementioned properties.

The study under statistical
settings is built upon the deterministic results in Appendix \ref{sec:det}, which could be of
independent interest. See Appendix \ref{sec:prooftheory} for this section's proofs. 

To analyze the estimation error under statistical settings, we assume $n_j = n$, $w_j = 1$, and
$\lambda_j = \lambda$ for all $j \in [m]$. In addition, assume that $\ell:~\RR^d \times \cX \to \RR$ is convex in its first argument and let $\bl: \RR^d \times \cX \rightarrow
\RR^d$ be a vector-valued function such that for every $\btheta \in \RR^d$ and $\bxi \in \cX$, $\bl(\btheta, \bxi)$ belongs to the subdifferential $\partial_{\btheta}
\ell(\btheta, \bxi)$ of $\ell$. We make the following assumptions.

\begin{assumption} [Concentration]\label{assump:conc} 
	There exists an absolute constant $c$ such that $\norm{\bl(\btheta,
	\bxi)}_{\psi_2} \le \sigma \leq c$ for all $\btheta \in \RR^d$, $\bxi \sim \cP_j$ and $j \in [m]$. 
\end{assumption}

\begin{assumption} [Regularity] \label{assump:reg}
Let $F_j(\btheta) = \mathbb{E}_{\bxi \sim \cP_j } \ell(\btheta, \bxi)$. Suppose that $F_j$ is twice differentiable on
$\btheta$, and denote $\bSigma_j(\btheta) =
\nabla^2 F_j(\btheta)$. There exists a constant $C_1 > 0$ such that
$$
\norm{\bSigma_j(\btheta_1) - \bSigma_j(\btheta_2)} \le C_1 \norm{\btheta_1 - \btheta_2}_2,~~
\forall \btheta_1, \btheta_2 \in \RR^d, ~ j \in [m].
$$
Furthermore, there exist $\btheta^* \in \RR^d$ and constants $c_1 \in (0, 1)$, $M > 0$ such that for any $\btheta \in B(\btheta^*, M)$ and $j \in [m]$, all eigenvalues of $\bSigma_j(\btheta)$ belong to $(c_1, 1 / c_1)$.
\end{assumption}

\begin{assumption}[Variability of $\bl$] \label{assump:smooth}
Assume one of the followings hold.
\begin{enumerate}
\item There is a function $U:~\cX \times \RR^d \times \RR^d \times \RR^d \to \RR$ such that
$$
\left|\bv^\top (\bl(\btheta_1, \bxi) - \bl(\btheta_2, \bxi))\right| \le U(\bxi, \bv, \btheta_1, \btheta_2)
\norm{\btheta_1 - \btheta_2}_2
$$
holds for all $\bv \in \RR^d$. $U(\bxi, \bv, \btheta_1, \btheta_2)$ satisfies that 
$$
\sup_{\norm{\bv}_2 = 1} \sup_{\btheta_1, \btheta_2} \mathbb{E}_{\bxi \sim \cP_j} \exp(t_0 U(\bxi, \bv, \btheta_1,
\btheta_2)) \le C,
$$
$$
\sup_{\norm{\bv}_2 = 1} \mathbb{E}_{\bxi \sim \cP_j} \sup_{\btheta_1, \btheta_2} U(\bxi, \bv, \btheta_1, \btheta_2) \le
d^{c_2},
$$
for some $c_2, t_0, C > 0$. Furthermore, assume that $\sup_{\norm{\bv}_2 = 1} \mathbb{E}_{\bxi \sim \cP_j} \exp(t_0 \left|\bv^\top \bl(\btheta^*, \bxi)\right|)
\le C$ holds for some constants $t_0, C >
0$. 
\item For some constants $c_2, c_3, r > 0$,
\begin{align*}
&\sup_{\btheta_1 \in B(\btheta^*, r)} \mathbb{E}_{\bxi \sim \cP_j} \sup_{\substack{\btheta_2 \in B(\btheta^*,
		r) \\ \btheta_2 \in B(\btheta_1, n^{-Z})}} \norm{\bl(\btheta_1, \bxi) - \bl(\btheta_2,
	\bxi)}_2^4 \\
& \le \frac{d^{c_2}}{n^{c_3 Z}} 
\end{align*}
for any large $Z > 0$. Also, 
\begin{align*}
	&\sup_{\norm{\bv}_2 = 1}
	 \mathbb{E}_{\bxi \sim \cP_j} \bigg\{ \bv^\top [ \bl(\btheta_1, \bxi) - \bl(\btheta_2, \bxi) ]^2 \\
	&\qquad \exp\{t_0 \left|\bv^\top [ \bl(\btheta_1, \bxi) -
	\bl(\btheta_2, \bxi) ] \right|\}  \bigg\} \le C \norm{\btheta_1 -
		\btheta_2}_2
\end{align*}
and $\sup_{\norm{\bv}_2 = 1} \mathbb{E}_{\bxi \sim \cP_j} \sup_{\btheta} \exp(t_0 \left|\bv^\top \bl(\btheta,
\bxi)\right|) \le C$ for some $t_0, C > 0$. 
\end{enumerate}
\end{assumption}

Assumption \ref{assump:conc} assumes that there exists a
subgradient of the loss function that is subgaussian. Assumption \ref{assump:reg} requires that the
Hessian of the population risk satisfies certain continuity condition and is bounded below and above
near its minimizer. 
Assumption \ref{assump:smooth} 
concerns the variability of the subgradient function $\bl$ 
which is standard in the literature of nonsmooth statistical learning \citep{CLZ21}. It is easy to verify that the examples satisfy the above assumptions under general regularity conditions on the covariates.

Note that, with regard to the empirical loss function, the assumptions above only concern
first-order conditions, which are weaker than the second-order condition required in \cite{DW22}, and thus they can apply to more general settings. In our case, this allows us to extend our analyses to non-smooth empirical loss functions. Further, our assumptions only
target \textit{one} subgradient of the empirical loss function. While a gradient may not always exist for the empirical loss function, a subgradient always exists. As it turns out, conclusion about
the closeness between one subgradient and its expectation is sufficient for conclusion on a uniform
closeness between all subgradients and their expectation. See Appendix \ref{sec:det} for more
details.


Define $\widetilde{\bTheta} = (\widetilde{\btheta}_1, \cdot\cdot\cdot, \widetilde{\btheta}_m)$ as
the estimators from STL, i.e., $\widetilde{\btheta}_j = \mathop{\mathrm{argmin}}_{\btheta \in
\RR^d} f_j(\btheta)$. We have the following result on the closeness between $\widehat{\btheta}_j$ and $\widetilde{\btheta}_j $, ensuring the former's fidelity to its associated dataset $\cD_j$. 

\begin{theorem}[Personalization] \label{thm:statperson}
Let Assumptions \ref{assump:conc}, \ref{assump:reg}, and \ref{assump:smooth} 
	hold. There exist constants $C_1$ and $C_2$ such that, when
	$\lambda < \rho M /4$, the following holds with probability at least $1 - C_1 n^{-d}$: 
$$
\norm{\widetilde{\btheta}_j - \btheta^*_j}_2 \le C_2 \sigma \sqrt{\frac{d \log n + \log m}{n}};
$$
$$
\norm{\widehat{\btheta}_j - \widetilde{\btheta}_j}_2 \le C_2 \left(\lambda + \sigma \sqrt{\frac{d \log n +
\log m}{n}}\right)
$$
for all $j \in [m]$.
\end{theorem}
Note that the output of $\widehat{\btheta}_j$ of (\ref{pro:vanilla}) always satisfies
\begin{equation}\label{eq:stlperturb}
	\widehat{\btheta}_j \in \mathop{\mathrm{argmin}}_{\btheta \in \RR^d} \left\{f_j(\btheta) +
	\lambda \norm{\btheta - \widehat{\bbeta}}_2\right\}, \;\;\;\; \forall j \in [m].
\end{equation}
Therefore, $\widehat{\btheta}_j$ and $\widetilde{\btheta}_j$ minimize similar functions. The penalty
term $\lambda \norm{\btheta - \widehat{\bbeta}}_2$ in (\ref{eq:stlperturb}) can be viewed as a
perturbation added to the objective function $f_j$, and Theorem \ref{thm:statperson} tells us that it
can only perturb the minimizer by a limited amount decided by the penalty
level $\lambda$. Intuitively, when the empirical loss
function $f_j$ is ``close" to a strongly convex function in a neighborhood of its minimizer
$\widetilde{\btheta}_j$, the Lipschitz penalty function does not make much difference. Theorem
\ref{thm:statperson} guarantees the fidelity of our approach \eqref{pro:vanilla} to individual datasets for general $M$-estimation.

By Assumption \ref{assump:conc}, we have $\sigma \lesssim 1$. Theorem \ref{thm:statperson} implies
that when $\lambda \lesssim \sigma \sqrt{\frac{d \log n + \log m}{n}}$, the bound
$\norm{\widehat{\btheta}_j - \btheta^*_j}_2 \lesssim \sigma \sqrt{\frac{d \log n + \log m}{n}}$
simultaneously holds for all $j \in [m]$ with high probability. In that case, our approach \eqref{pro:vanilla} achieves the same
estimation error rate of STL up to logarithmic factors. 

In the definition and theorem to follow, we consider all the tasks and study the
adaptivity and robustness of (\ref{pro:vanilla}).  

\begin{assumption} [Task Relatedness] \label{assump:relate}
	There exist $\varepsilon,\, \delta \ge 0$ and subset $S \subseteq [m]$ such that
$$
\left|S^c\right|  \le \varepsilon m \;\;\;\;\text{and}\;\;\;\; \min_{\btheta \in \RR^d} \max_{j \in S}  \norm{\btheta^*_j - \btheta}_2  \le \delta.
$$
\end{assumption}
It is worth pointing out that any $m$ tasks are $(0, \max_{j \in
[m]} \norm{\btheta_j^*}_2 )$-related. Smaller $\varepsilon$ and
$\delta$ imply stronger similarity among the tasks. When all but a small proportion of
$\{\btheta_j^*\}_{j=1}^m$ are close to each other, the following theorem shows that a single choice
of $\lambda$ can automatically enforce an appropriate degree of relatedness among the learned
models, while tolerating a reasonable fraction of exceptional tasks that are dissimilar to others. 

\begin{theorem}[Adaptivity] \label{thm:statvanilla}
Let Assumptions \ref{assump:conc}, \ref{assump:reg}, \ref{assump:smooth} and \ref{assump:relate} hold. 
	 When $ n^{d(m - 1)}m^{m - d} \geq 1$, there exist positive
	constants $\{C_i\}_{i=0}^4$ such that, if
$$
C_1 \sigma\sqrt{\frac{d \log n + \log m}{n}} < \lambda < C_2 \sigma
$$
and $0 \le \varepsilon < C_3$, the following bounds hold with probability at least $1 - C_4(m^{-d}+1)n^{-d}$:
\begin{align*}
&\max_{j \in S}
\norm{\widehat{\btheta}_j - \btheta_j^*}_2  \le C_0 \bigg(\sigma\sqrt{\frac{d \log (mn)}{mn}} + \min
\left\{\lambda, \delta\right\} + \varepsilon \lambda\bigg) , \\
&\max_{j \in S^c} \norm{\widehat{\btheta}_j - \btheta_j^*}_2  \le C_0 \lambda.
\end{align*}
Moreover, there exists a constant $C_5$ such that under the conditions $\varepsilon = 0$ and $C_5
\delta \sqrt{n} < \sigma\sqrt{d \log n + \log m}$, we have $\widehat{\btheta}_1 = \cdot \cdot \cdot = \widehat{\btheta}_m
= \mathop{\mathrm{argmin}}_{\btheta \in \RR^d} \{\sum_{j=1}^{m} f_j(\btheta)\}$.
\end{theorem}

It is worth pointing out that the same error bound on $\max_{j \in S} \norm{\widehat{\btheta}_j - \btheta_j^*}_2$ holds even if the data of the tasks in $S^c$ have been arbitrarily contaminated.  

Theorem \ref{thm:statvanilla} simultaneously controls the estimation errors for all individual tasks
in $S$ and suggests choosing $\lambda = C \sigma \sqrt{\frac{d \log n + \log m}{n}}$ for some
constant $C$. In practice, this $C$ can be selected by cross-validation. This allows
us to choose a single $\lambda$ to achieve minimax optimality (up to a logarithmic factor), matching the minimax lower bound in \cite{DW22}. Indeed, that lower bound is proved for a class of smooth losses that are included by our general function classes.

For any $\varepsilon$ and $\delta$, a simple bound $\norm{\widehat{\btheta}_j - \btheta^*_j}_2
\lesssim \lambda$ always holds for all $j \in [m]$, which echoes Theorem \ref{thm:statperson}. In
comparison, Theorem \ref{thm:statvanilla} implies more refined results. When $\delta = \varepsilon =
0$, all target parameters are the same and $S = [m]$. Data pooling becomes a natural approach, whose
error rate is $O(\sigma \sqrt{d / mn})$. Our approach \eqref{pro:vanilla} has the same rate (up to a logarithmic factor). When
$\varepsilon = 0$ ad $\delta$ grows from 0 to $+\infty$, the error bounds smoothly transit from
those for data pooling to those for STL. See Section \ref{sec:exp} for illustration.  

The second term $\min \{\lambda, \delta\}$ is non-decreasing in the discrepancy $\delta$ among
$\{\btheta_j^*\}_{j \in S}$. It increases first and then flattens out, never exceeding the error
rate of STL. Combined with the first term, when $\varepsilon = 0$, our approach \eqref{pro:vanilla} achieves the smaller error
rate between data pooling and STL. When $\varepsilon > 0$, the third term $\varepsilon\lambda$ is
the price we pay for not knowing the index set $S^c$ for outlier tasks. 

In summary, the theoretical investigation yields a principled approach for choosing a single
regulatization parameter $\lambda$ for all tasks. The resulting estimators automatically adapt to
unknown task relatedness and are robust against a certain fraction of outlier tasks.

\section{NUMERICAL EXPERIMENTS}\label{sec:exp}

We conduct experiments on synthetic and real data to test our approach in various scenarios. Below we present descriptions and key findings. 
The curves and error bands show the means and their $95\%$ confidence intervals computed from 100 independent runs, respectively.

\subsection{Synthetic Data}

We first generate synthetic data for multi-task quantile regression. The number of tasks is $m = 50$. For every $j \in [m]$, the dataset $\cD_j$ consists of $n = 200$ samples $\{ (\bx_{ji}, y_{ji}) \}_{i=1}^n$. The covariate vectors $\{ \bx_{ji} \}_{(i, j) \in [n] \times [m]}$ are i.i.d.~from the 20-dimensional standard normal distribution, given which we sample each response $y_{ji} = \bx_{ji}^{\top} \bgamma^{\star}_j + \varepsilon_{ji}$ from a linear model with noise term $\varepsilon_{ji} \sim N(0,0.25)$ being independent of the covariates. The coefficient vectors $\{ \bgamma^{\star}_{j} \}_{j=1}^m \subseteq \RR^{20}$ are generated according to the prescribed level of task relatedness defined in Assumption \ref{assump:relate}. For every $\varepsilon \in \{ 0, 0.1 \}$ and $\delta \in \{ 0, 0.2, 0.4, \cdots, 2 \}$ we use the procedure below to get $m$ tasks that are $(\varepsilon,\delta)$-related and share the same signal strength.

\begin{itemize}
\item Select $\varepsilon m$ tasks uniformly at random and let $S$ be the index set of \emph{unselected} tasks;
\item Draw $m$ i.i.d.~random vectors $\{ \bm{\eta}_j \}_{j=1}^m$ uniformly from the unit sphere, and set $\bgamma^{\star}_j = 2\bm{\eta}_j$ for all $j \notin S$;
\item For each $j \in S$, set $\bgamma^{\star}_j = (2\cos \alpha) \be_1 + (2\sin \alpha ) \bm{\eta}_j$, where $\be_1 = (1,0,\cdots,0)$ and $\alpha = 2 \arcsin (\delta / 4)$.
\end{itemize}
We have $\| \bgamma^{\star}_j - 2 \be_1 \|_2 = \delta$, $\forall j \in S$ and $\| \bgamma^{\star}_j \|_2 = 2$, $\forall j \in [m]$.

Our target quantile level is $\tau = 0.9$. Given the covariates $\bx_{ji}$, the $\tau$-th quantile of the response $y_{ji}$ is $\bx_{ji}^{\top} \bgamma^{\star}_j + 0.5 \Phi^{-1}(\tau) $, where $\Phi$ is the cumulative distribution function of $N(0,1)$. In quantile regression, we add an all-one covariate and enlarge the input dimension to $d = 21$. For the $j$-th task, the true coefficients in the quantile function are $ \btheta^{\star}_j = ( 0.5 \Phi^{-1}(\tau), \bgamma^{\star \top}_j  )^{\top} \in \RR^{21}$. For any algorithm that produces estimates $\{ \widehat{\btheta}_j \}_{j=1}^{m} $ of $\{  \btheta^{\star}_j \}_{j=1}^m$, we compute the maximum estimation error $\max_{j \in [m]} \| \widehat{\btheta}_j  - \btheta^{\star}_j \|_2$ and its restricted version $\max_{j \in S} \| \widehat{\btheta}_j  - \btheta^{\star}_j \|_2$ on the subset $S$ containing similar tasks (if $\varepsilon > 0$).

Following our theories, we set the regularization parameter $\lambda$ for our approach to be $C \sqrt{d / n}$ and select $C$ from $\{ 0.1, 0.2, \cdots, 1 \}$ by 5-fold cross-validation. We compare the approach with single-task learning (STL) and data pooling (DP). Figures \ref{fig-epsilon_0} and \ref{fig-epsilon_2} demonstrate how the estimation errors grow with the heterogeneity parameter $\delta$. 

\begin{figure}[!h]
	\centering
	\includegraphics[width=0.4\linewidth]{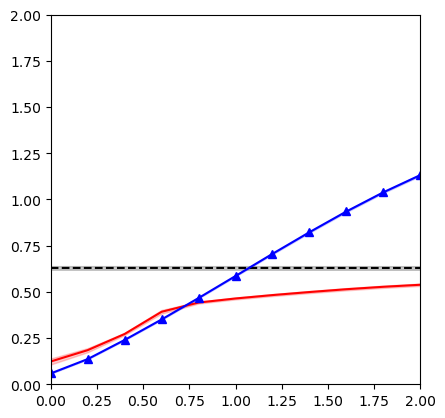}   
	\caption{Impact of task relatedness when $\varepsilon = 0$. $x$-axis: $\delta$. $y$-axis: $\max_{j \in [m]} \| \widehat{\btheta}_j - \btheta^{\star}_j \|_2 $. Red solid line: new approach. Blue triangles: DP. Black dashed line: STL.}\label{fig-epsilon_0}
\end{figure}

\begin{figure}[!h]
	\centering
	\subfigure{
		\includegraphics[width=0.4\linewidth]{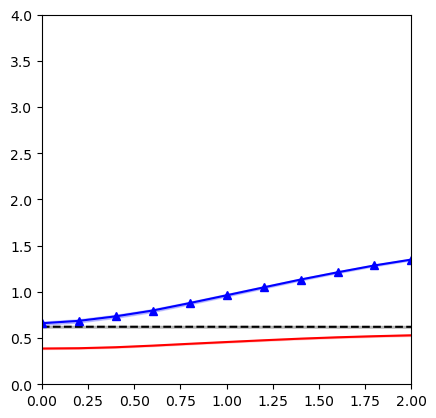}   
	}
	\subfigure{
		\includegraphics[width=0.4\linewidth]{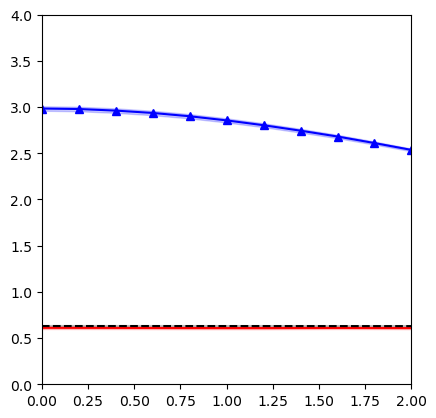}   
	}
	\caption{Impact of task relatedness when $\varepsilon = 0.2$. $x$-axis: $\delta$. $y$-axis: $\max_{j \in S} \| \widehat{\btheta}_j - \btheta^{\star}_j \|_2 $ (left) or $\max_{j \in [m]} \| \widehat{\btheta}_j - \btheta^{\star}_j \|_2 $ (right). Red solid lines: new approach. Blue triangles: DP. Black dashed lines: STL.}\label{fig-epsilon_2}
\end{figure}

The simulations confirm the theoretical guarantees Theorem \ref{thm:statvanilla} for our proposed method. When $\varepsilon = 0$ and $\delta$ is small, it behaves similarly as DP. As $\delta$ increases, the new method tackles the heterogeneity and never underperform STL, while DP's estimation error grows rapidly. When $\varepsilon=0.2$, the new method works well on the set $S$ of related tasks, and DP makes huge errors due to the $\varepsilon$-fraction of exceptional tasks. The two panels of Figure \ref{fig-epsilon_2} imply that the new method behaves similarly as STL on $S^c$. This agrees with Theorem \ref{thm:statperson} that our estimates for individual tasks are never too far from the corresponding empirical loss minimizers. Data pooling performs poorly on $S^c$.

\subsection{Real Data}


We apply the proposed method to a data-driven newsvendor problem. 
We use a real-world dataset made publicly available by \cite{BPS22} that contains sales data at 35 different stores in a local bakery chain over a period of 1215 days, from January 2016 to April 2019.
According to the authors, every evening each store orders products to be delivered the next morning from a central factory. Unsold goods will be disposed of at the end of the day. The authors use the sales data as the demand because all products are everyday items with typically high stock levels, which makes censored demand unlikely. The dataset also contains information about the weather, promotions, holidays, calendric (e.g. year, month, weekday) and lag features (e.g. mean demand over the past 7 days). 
To utilize the features we generalize the classical newsvendor problem in Example \ref{subsec:newsvendor} into a covariate-assisted data-driven newsvendor problem. 

In particular, recall that $\{ D_i \}_{i=1}^n \subseteq \RR$ are the realized daily market demands and $\{ \bx_i \}_{i=1}^n \subseteq \RR^d$ are the covariates for the corresponding days. Suppose we want to decide the ordering quantity $q_i$ using a linear combination of the $d$-dimension features, as $q_i=\bx_i^\top \btheta$ with a coefficient parameter $ \btheta \in\mathbb{R}^d$ to be determined. 
Assume that any leftover at the end of the day leads to a holding cost of $\$ h$ per unit. Meanwhile, any demand that cannot be satisfied results in a backorder cost of $\$ b$ per unit. The cost on the $i$-th day is $h (D_i - q_i)^+ + b (D_i  - q_i)^-$ dollars, which is proportional to the check loss $\rho_{\tau} ( D_i - q_i )$ with $\tau = b / (b + h)$. 
We can estimate the best $\btheta$ with the minimum expected cost through minimizing the following nonsmooth objective function,
$$
f( \btheta) = \frac{1}{n} \sum_{i=1}^{n} \left[  b (D_i - x_i^\top  \btheta)^+ + h(D_i - x_i^\top  \btheta)^- \right].
$$
It can also be viewed as a quantile regression problem with $\tau=b/(b+h)$. 

We study the first product in the dataset, which is sold at $m = 32$ stores. Each store needs a model that decides its order quantity every day to minimize the cost. Throughout our experiments, we fix $\tau = 0.9$.
For every $j \in [m]$, the $j$-th 
store has historical data $\{ (\bx_{ji} , D_{ji}) \}_{i=1}^n$, where $\bx_{ji}$ consists of real-valued covariates available before the $i$-th day, and $D_{ji}$ is the demand on that day. We use 19 covariates and add an all-one covariate. Therefore, $\bx_{ji}$ has dimension $d = 20$. We focus on linear decision rules of the form $\bx_{ji} \mapsto \bx_{ji}^{\top} \btheta_j$, where $\btheta_j \in \RR^d$. The problem is formulated as multi-task quantile regression. The loss function of task $j$ is $f_j (\btheta) = \frac{1}{n} \sum_{i=1}^{n} \rho_{\tau} ( D_{ji} - \bx_{ji}^{\top} \btheta ) $. Same as our experiments on synthetic data, here we also compare the new method with single-task learning (STL) and data pooling (DP). 

Our testing set consists of all the data in 2019 (four months). For each $k \in [12]$, we implement all methods on the data over the $k$ months before 2019. The penalty parameter for our new method is $\lambda = C \sqrt{d / n}$ with $C \in \{ 0.1, 0.2, \cdots, 1 \}$. We run the method on the first $80\%$ of the training data for each $C$ and evaluate them on the rest $20\%$. Then, we choose the one with the lowest validation error, refit the models on the whole training set. We measure the performance of three methods by their average testing losses over all the $m$ tasks, which are proportional to the average daily costs of those decision rules.

Figure \ref{fig-newsvendor} reveals how the testing losses decrease as more training data become available. In particular, the new method is always the best. When there are only one or two month's data for training, both the new method and DP outperform STL. Then the curve of DP flattens out, as its model misspecification error dominates the statistical error. The new method and STL benefit from increased sample size. The former is significantly better by a large margin when there are at most 8 months' data for training. The two approaches have little difference when the training set is sufficiently large. Therefore, our approach is always a good choice, especially when the data are scarce.

\begin{figure}[!h]
	\centering
	\includegraphics[width=0.4\linewidth]{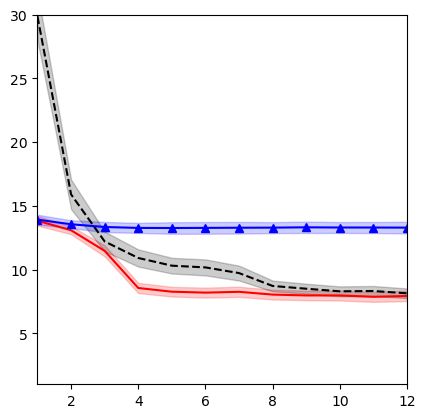}   
	\caption{Impact of sample sizes. $x$-axis: number of months for training. $y$-axis: average testing loss. Red solid line: new approach. Blue triangles: DP. Black dashed line: STL.}\label{fig-newsvendor}
\end{figure}

\section{DISCUSSIONS}

We have studied a simple approach for multi-task optimization problems with possibly nonsmooth loss, theoretically proved its adaptivity to the unknown task relatedness, and demonstrated its power on real data. There are several directions we plan to pursue in future research. We will develop efficient algorithms for solving the multi-task non-smooth optimization problems we studied. The algorithms should fit for distributed computing architectures and preserve the privacy of individual dataset owners. We will also develop statistical tools for uncertainty quantification in the multi-task setting. 


\newpage 
\appendix

\section{DETERMINISTIC RESULTS} \label{sec:det}
In this section, we present deterministic results on (\ref{pro:vanilla}); see Appendix \ref{sec:proofdet} for
the proofs.

\begin{definition}[Regularity] \label{def:reg}
	Let $\btheta^* \in \RR^d$, $0 < M \le \infty$, $0 < \rho \le L < \infty$, $F:~\RR^d \to \RR$, and $0 \le \zeta < \infty$. A convex function $f:~\RR^d \to \RR$ is said to be $(\btheta^*, M, \rho, L, F, \zeta)$-regular if 
	\begin{itemize}
		\item $F$ is convex and twice differentiable;
		\item $\rho \bI \preceq \nabla^2 F(\btheta) \preceq L \bI$ holds for all $\btheta \in
		B(\btheta^*, M)$;
		\item $\sup_{\btheta \in B(\btheta^*, M)} \norm{\bbf -
			\nabla F(\btheta)}_2 \le \zeta \le \rho M /2$ where $\bbf: \RR^d \rightarrow \RR^d$ is such that for
		every $\btheta \in \RR^d$, $\bbf(\btheta) \in \partial f(\btheta)$;
		\item $\norm{\nabla F(\btheta^*)}_2 \le \rho M / 2 - \zeta$.
	\end{itemize}
\end{definition}

\begin{theorem}[Personalization] \label{thm:detperson}
	If $f_j$ is $(\btheta^*_j, M, \rho, L, F_j, \zeta_j)$-regular and $0 \le \lambda < \rho M / 2 -
	\zeta_j$,
	then
	\begin{align*}
	&\norm{\widehat{\btheta}_j - \widetilde{\btheta}_j}_2 \le
	\frac{\lambda}{\rho} + \frac{\zeta_j}{\rho},\\
	&\norm{\widetilde{\btheta}_j - \btheta^*_j}_2 \le \frac{\norm{\nabla F(\btheta^*_j)}_2}{\rho} +
	\frac{\zeta_j}{\rho}.
	\end{align*}
\end{theorem}
In most cases, $\norm{\nabla F(\btheta^*_j)}_2 = 0$ since the $F_j$ that
we relate $f_j$ with would be the population loss function and $\btheta^*_j$ is the population
loss minimizer. As such, the optimality gap of STL is due to $\zeta_j$, a uniform first-order
upper-bound.

\begin{definition}[Task relatedness] \label{def:relate}
	Let $\varepsilon, \;\delta, \; \zeta_S \ge 0$, $\{\btheta^*_j\}_{j=1}^{m} \subseteq \RR^d$, $0 < M \le \infty$,
	$0 < \rho \le L < +\infty$, and $S \subseteq [m]$. $\{f_j\}_{j=1}^m$ are said to be $(\varepsilon, \delta,
	\zeta_S)$-related with regularity parameters $(\{\btheta^*_j\}_{j=1}^{m}, M, \rho, L, S, \{F_j\}_{j \in
		S}, \{\zeta_j\}_{j \in S})$ if
	\begin{itemize}
		\item for any $j \in S$, $f_j$ is $(\btheta^*_j, M, \rho, L, F_j, \zeta_j)$-regular
		(Definition \ref{def:reg});
		\item $\min_{\btheta \in \RR^d} \max_{j \in S} \{\norm{\btheta^*_j - \btheta}_2\} \le
		\delta$;
		\item $ \left|S^c\right| / \left|S\right| \le \varepsilon$;
		\item $\sup_{\btheta \in B(\btheta^*, M)} \norm{\bg_S(\btheta) -
			\nabla G_S(\btheta)}_2 \le \zeta_S \le \sum_{j \in S} \zeta_j$, where $G_S = \sum_{j \in S} F_j$
		and $\bg_S: \RR^d \rightarrow \RR^d$ is such that for every $\btheta \in \RR^d$, $\bg_S(\btheta)
		\in \partial \sum_{j \in S} f_j(\btheta)$. 
	\end{itemize}
\end{definition}

\begin{theorem}[Adaptivity and Robustness] \label{thm:detvanilla}
	Let $\{f_j\}_{j=1}^m$ be $(\varepsilon, \delta,
	\zeta_S)$-related with regularity parameters $(\{\btheta^*_j\}_{j=1}^{m}, M, \rho, L, S,$ $\{F_j\}_{j \in
		S}, \{\zeta_j\}_{j \in S})$. Define $
	\kappa = L / \rho$. Suppose $\kappa \varepsilon < 1$
	and
	\begin{equation}\label{cond:vanillalambda}
	\frac{5 \kappa}{1 - \kappa \varepsilon} \max_{j \in S} \{
	(\norm{\nabla F_j (\btheta^*_j)}_2 + \zeta_j)\} < \lambda < \frac{\rho M}{2}.
	\end{equation}
	Then, the estimators $\{\widehat{\btheta}_j\}_{j \in S}$ in (\ref{pro:vanilla}) satisfy
	\begin{align*}
	&\norm{\widehat{\btheta}_j - \btheta^*_j}_2 \le \frac{\norm{\sum_{k \in S} \nabla F_k
			(\btheta^*_k)}_2}{\rho \left|S\right|} \\ &\;\;\;\;\;\;\;\; + \frac{7}{(1 - \kappa
		\varepsilon)} \min \left\{3 \kappa^2 \delta, \frac{2\lambda}{5\rho}\right\} +
	\frac{\varepsilon \lambda}{\rho} + \frac{2 \zeta_S}{\rho \left|S\right|}.
	\end{align*}
	Moreover, there exists a constant $C$ such that under the conditions $\varepsilon = 0$ and $C \kappa
	L \delta < \lambda$, we have $\widehat{\btheta}_1 = \cdot \cdot \cdot = \widehat{\btheta}_m
	= \mathop{\mathrm{argmin}}_{\btheta \in \RR^d} \{\sum_{j=1}^{m} f_j(\btheta)\}$. 
\end{theorem}

\section{PROOF OF DETERMINISTIC RESULTS} \label{sec:proofdet}

\subsection{Proof of Theorem \ref{thm:detperson}}
By Lemma \ref{lem:tech1} and $\norm{\nabla F_j(\btheta^*_j)}_2 \le \rho M / 2 - \zeta_j $,
any minimizer $\widetilde{\btheta}_j$ of $f_j$ satisfies $\norm{\widetilde{\btheta}_j -
	\btheta^*_j}_2 \le \norm{\nabla F_j(\btheta^*_j)}_2 / \rho + \zeta_j / \rho < M / 2$. Hence $\nabla^2
F_j(\btheta) \succeq \rho \bI$ holds for all $\btheta \in B(\widetilde{\btheta}_j, M / 2)$. By Lemma
\ref{lem:tech2} and $\lambda < \rho M / 2 - \zeta_j$, $\norm{\widehat{\btheta}_j -
	\widetilde{\btheta}_j}_2 \le \lambda / \rho + \zeta_j / \rho$. 

\subsection{Proof of Theorem \ref{thm:detvanilla}}
First, note that the task-relatedness, combined with Lemma \ref{lem:tech5}, yields 
$$
\sup_{\btheta \in B(\btheta^*, M)} \sup_{\bg_S \in \partial g_S(\btheta)} \norm{\bg_S -
	\nabla G_S(\btheta)}_2 \le \zeta_S,
$$
$$
\sup_{\btheta \in B(\btheta^*, M)} \sup_{\bbf_j \in \partial f_j (\btheta)}\norm{\bbf_j -
	\nabla F_j(\btheta)}_2 \le \zeta_j, \;\;\;\;\;\; \forall j \in S,
$$
where $g_S = \sum_{j \in S} f_j$ and $G_S = \sum_{j \in S} F_j$.

Define $\gamma = \max_{j \in S} \{
\norm{\nabla F_j (\btheta^*_j)}_2 + \zeta_j\}$. We first assume 
\begin{equation} \label{assump:biglambda}
3L\delta < \gamma + \frac{1 - \kappa \varepsilon}{5
	\kappa}\lambda.
\end{equation}
Define $\btheta^* = \mathop{\mathrm{argmin}}_{\btheta \in \RR^d} \max_{j \in S} \norm{\btheta^*_j -
	\btheta}_2$, and recall that $\min_{\btheta \in \RR} \max_{j \in S} \{\norm{\btheta^*_j - \btheta}_2\}
\le \delta$. From (\ref{cond:vanillalambda}) we have that $\gamma < \rho M /10$ and $\lambda < \rho M / 2$. When
$$
3 L \delta < \gamma + \frac{1 - \kappa \varepsilon}{5 \kappa} \lambda < \gamma + \lambda <
\frac{3}{5} L M,
$$
we have $\delta < M / 5$; thus $\max_{j \in S} \norm{\btheta^*_j -
	\btheta^*}_2 < M / 5$. Recall that, for any $j \in S$, we have the sub-regularity condition that
$\nabla^2 F(\btheta) \preceq L \bI$, $\forall \btheta \in B(\btheta^*_j, M)$; thus we have $\nabla^2
F(\btheta) \preceq L \bI$, $\forall \btheta \in B(\btheta^*, 4M/5)$.  This leads to $\norm{\nabla F_j(\btheta_j^*) - \nabla
	F_j(\btheta^*)}_2 \le L \delta$, $\forall j \in S$. Define $\eta = \max_{j \in S}
\left\{\norm{\nabla F_j (\btheta^*)}_2 + \zeta_j\right\}$. By triangle inequality,  
$$
\eta \le \gamma + \max_{j \in S} \{ \norm{\nabla F_j(\btheta_j^*) - \nabla
	F_j(\btheta^*)}_2\} \le \gamma + L\delta < \frac{4}{3} \gamma + \frac{1 - \kappa
	\varepsilon}{15 \kappa}\lambda,
$$
where the last inequality results from (\ref{assump:biglambda}). Consequently, we have
$$
\frac{3\kappa \eta}{1 - \kappa \varepsilon} < \frac{4 \kappa \gamma}{1 - \kappa
	\varepsilon} + \frac{\lambda}{5} < \lambda < \frac{\rho M}{2} < LM.
$$
By Lemma \ref{lem:vanillaprep}, $\widehat{\btheta}_j = \widehat{\bbeta}$ for all $j \in S$ and 
$$
\norm{\widehat{\bbeta} - \btheta^*}_2 \le \frac{\varepsilon \lambda}{\rho} + \frac{\norm{\sum_{j \in S} \nabla F_j
		(\btheta^*)}_2}{\rho \left|S\right|} + \frac{2\zeta_S}{\rho \left|S\right|}.
$$
For any $j \in S$, $\norm{\widehat{\btheta}_j - \btheta_j^*}_2 \le \norm{\widehat{\bbeta} -
	\btheta^*}_2 + \norm{\btheta^* - \btheta^*_j} \le \norm{\widehat{\bbeta} -
	\btheta^*}_2 + \delta$. Also, 
$$
\frac{\norm{\sum_{j \in S} \nabla F_j
		(\btheta^*)}_2}{\rho \left|S\right|} \le \frac{\norm{\sum_{j \in S} \nabla F_j
		(\btheta^*_j)}_2}{\rho \left|S\right|} + \frac{L \delta}{\rho} \le \frac{\norm{\sum_{j \in S} \nabla F_j
		(\btheta^*_j)}_2}{\rho \left|S\right|} + \kappa\delta.
$$
Based on the above estimates and noting that $\kappa \ge 1$, we have
\begin{equation*}
\norm{\widehat{\btheta}_j - \btheta^*_j}_2 \le \frac{\norm{\sum_{k \in S} \nabla F_k
		(\btheta^*_k)}_2}{\rho \left|S\right|} + 2 \kappa \delta +
\frac{\varepsilon \lambda}{\rho} + \frac{2 \zeta_S}{\rho \left|S\right|}, \;\;\;\;\;\; \forall j \in S.
\end{equation*}

Now, note that condition (\ref{cond:vanillalambda}) forces $\lambda > 5 \kappa \gamma$. The
above result implies that when $3L\delta < \gamma + \frac{1 - \kappa \varepsilon}{5
	\kappa}\lambda$, for any $j \in S$,  
$$
\norm{\widehat{\btheta}_j - \btheta^*_j}_2 \le \frac{\norm{\sum_{k \in S} \nabla F_k
		(\btheta^*_k)}_2}{\rho \left|S\right|} + \frac{1}{\rho} \min \left\{3 L
\delta, \gamma + \frac{1 - \kappa \varepsilon}{5 \kappa} \lambda\right\} +
\frac{\varepsilon \lambda}{\rho} + \frac{2 \zeta_S}{\rho \left|S\right|}.
$$
On the other hand, when $3L\delta \ge \gamma + \frac{1 - \kappa \varepsilon}{5
	\kappa}\lambda$, we use Theorem \ref{thm:detperson} to get
$$
\norm{\widehat{\btheta}_j - \btheta^*_j}_2 \le \frac{2\gamma + \lambda}{\rho} \le
\left(\frac{2}{5\kappa} + 1\right) \frac{\lambda}{\rho} \le \frac{7 \lambda}{5 \rho},
$$
where the second inequality is due to condition (\ref{cond:vanillalambda}). Denote
$$
U = \frac{\norm{\sum_{k \in S} \nabla F_k
		(\btheta^*_k)}_2}{\rho \left|S\right|} + \frac{\varepsilon \lambda}{\rho} + \frac{2 \zeta_S}{\rho
	\left|S\right|}.
$$
To summarize, for any $j \in S$ we
have
\begin{align*}
\norm{\widehat{\btheta}_j - \btheta^*_j}_2 &\le \frac{\norm{\sum_{k \in S} \nabla F_k
		(\btheta^*_k)}_2}{\rho \left|S\right|} + \frac{\varepsilon \lambda}{\rho} + \frac{2 \zeta_S}{\rho
	\left|S\right|} + \frac{1}{\rho} \cdot \frac{7
	\kappa}{1 - \kappa \varepsilon} \min \left\{3 L
\delta, \gamma + \frac{1 - \kappa \varepsilon}{5 \kappa} \lambda\right\} \\ & \le \frac{\norm{\sum_{k \in S} \nabla F_k
		(\btheta^*_k)}_2}{\rho \left|S\right|} + \frac{\varepsilon \lambda}{\rho} + \frac{2 \zeta_S}{\rho
	\left|S\right|}
+ \frac{1}{\rho} \cdot \frac{7 \kappa}{1 - \kappa \varepsilon} \min \left\{3 L
\delta, \frac{2\lambda}{5 \kappa} \right\} \\ & \le \frac{\norm{\sum_{k \in S} \nabla F_k
		(\btheta^*_k)}_2}{\rho \left|S\right|} + \frac{\varepsilon \lambda}{\rho} + \frac{2 \zeta_S}{\rho
	\left|S\right|} + \frac{7}{1 - \kappa
	\varepsilon} \min \left\{3 \kappa^2 \delta, \frac{2\lambda}{5\rho}\right\}.
\end{align*}
The relation between $\widehat{\btheta}_j$ and $\mathop{\mathrm{argmin}}_{\btheta \in \RR^d}
\{\sum_{j=1}^{m} f_j(\btheta)\}$ can be derived from Lemma \ref{lem:vanillaprep}.

\subsection{Supporting Lemmas for Deterministic Results}
\begin{lemma}\label{lem:vanilladp}
	Let $\left\{f_j\right\}_{j=1}^{m}$ and $\left\{F_j\right\}_{j=1}^{m}$ be convex. Suppose $F_j$
	is twice differentiable for all $j \in [m]$, and there exist $\btheta^* \in \RR^d$ and $0 < M <
	\infty$ such that for all $j \in [m]$ 
	$$
	\rho_j \bI \preceq \nabla^2 F_j(\btheta) \preceq L_j \bI, \;\;\;\;\;\; \forall \btheta \in
	B(\btheta^*, M) 
	$$
	with some $0 < \rho_j \le L_j < +\infty$. Define
	$$
	f_0 = \sum_{j=1}^{m} f_j \square (\lambda_j \left\|\cdot\right\|_2), \;\;\;\;\;\;g_0 =
	\sum_{j=1}^{m} f_j,\;\;\;\;\;\; G_0 = \sum_{j=1}^{m} F_j,
	$$
	and denote $\rho_0 =
	\sum_{j=1}^{m} \rho_j$. If, for some $0 \le \zeta_0,\, \zeta_1,\, ...,\, \zeta_m < +\infty$, 
	$$
	\sup_{\btheta \in B(\btheta^*, M)} \sup_{\bg_0 \in \partial g_0(\btheta)} \norm{\bg_0 -
		\nabla G_0(\btheta)}_2 \le \zeta_0,
	$$
	$$
	\sup_{\btheta \in B(\btheta^*, M)} \sup_{\bbf_j \in \partial f_j (\btheta)}\norm{\bbf_j -
		\nabla F_j(\btheta)}_2 \le \zeta_j, \;\;\;\;\;\; \forall j \in [m],
	$$
	and 
	$$
	\norm{\nabla F_j(\btheta^*)}_2 + \zeta_j + \frac{2 L_j}{\rho_0} \left(\bignorm{\sum_{k=1}^{m} \nabla F_k(\btheta^*)}_2 +
	\zeta_0 \right) < \lambda_j < \norm{\nabla F_j(\btheta^*)}_2 + \zeta_j + L_j M
	$$
	for all $j \in [m]$, then $\widehat{\btheta}_1 = \cdot\cdot\cdot = \widehat{\btheta}_m = \widehat{\bbeta} = \widetilde{\btheta}$, 
	\begin{align*}
	&f_0(\btheta) = g_0(\btheta), \;\;\;\;\;\; \forall \btheta \in B(\btheta^*, R), \\
	&\norm{\widehat{\bbeta} - \btheta^*}_2 \le \frac{\norm{\sum_{k=1}^{m} \nabla F_k
			(\btheta^*)}_2}{\rho_0} + \frac{\zeta_0}{\rho_0}
	\end{align*}
	where $R = \min_{j \in [m]} \left\{(\lambda_j - \norm{\nabla F_j(\btheta^*)}_2 - \zeta_j) / L_j\right\}$. 
\end{lemma}

\begin{proof}[Proof of Lemma \ref{lem:vanilladp}]
	By assumption, $(\lambda_j - \norm{\nabla F_j(\btheta^*)}_2 - \zeta_j) / L_j < M$ for any $j \in [m]$. 
	Since $\lambda_j > \norm{\nabla F_j(\btheta^*)}_2 + \zeta_j$, by Lemma \ref{lem:tech3}, we have 
	$f_j = f_j \square (\lambda_j \norm{\cdot}_2)$ in 
	$B(\btheta^*, (\lambda_j - \norm{\nabla F_j(\btheta^*)}_2 - \zeta_j)$ $/L_j)$. Then, $f_0 = g_0$ in $B(\btheta^*, R)$.
	
	Since $\nabla^2 F_j(\btheta) \succeq
	\rho_j \bI$ for any $\btheta \in B(\btheta^*, M)$, we have 
	$$
	F_j(\btheta) - F_j(\btheta^*) \ge \frac{\rho_j}{2} \norm{\btheta - \btheta^*}_2^2 +
	\inner{\nabla F_j(\btheta^*)}{\btheta - \btheta^*},
	$$
	and thus
	\begin{align*}
	G_0(\btheta) - G_0(\btheta^*) &= \sum_{j=1}^{m} F_j(\btheta) - F_j(\btheta^*) \\ &\ge
	\frac{1}{2}\left(\sum_{j=1}^{m} \rho_j\right) \norm{\btheta - \btheta^*}_2^2 +
	\inner{\sum_{j=1}^{m} \nabla F_j(\btheta^*)}{\btheta - \btheta^*} \\ &=
	\frac{\rho_0}{2} \norm{\btheta - \btheta^*}_2^2 + \inner{\nabla G_0(\btheta^*)}{\btheta -
		\btheta^*}, \;\;\;\;\;\; \forall \btheta \in B(\btheta^*, M),
	\end{align*}
	from which we have $\nabla^2 G_0(\btheta^*) \succeq \rho_0 \bI$ for all $\btheta \in B(\btheta^*,
	M)$. By assumption, we have
	$$
	\frac{2}{\rho_0}\left(\bignorm{\sum_{k=1}^{m} \nabla F_k(\btheta^*)}_2 + \zeta_0 \right)
	\le \frac{\lambda_j - \norm{\nabla F_j(\btheta^*)}_2 - \zeta_j}{L_j},
	\;\;\;\;\;\; \forall j \in [m].
	$$
	Taking the minimum over $j$ on both sides, we have
	$$
	\bignorm{\sum_{j=1}^{m} \nabla F_j(\btheta^*)}_2 \le \frac{1}{2} R \rho_0 - \zeta_0.
	$$
	By Lemma \ref{lem:tech1}, 
	$$
	\widehat{\bbeta} = \mathop{\mathrm{argmin}}_{\bbeta \in \RR^d} f_0(\bbeta) =
	\mathop{\mathrm{argmin}}_{\btheta \in \RR^d} g_0(\btheta) = \widetilde{\btheta} \subseteq B\left(\btheta^*,
	\frac{\norm{\sum_{j=1}^{m} \nabla F_j (\btheta^*)}_2}{\rho_0} +
	\frac{\zeta_0}{\rho_0}\right).
	$$
	Finally, $\widehat{\btheta}_j = \widehat{\bbeta}$ follows from $\widehat{\btheta}_j \in
	\mathop{\mathrm{argmin}}_{\btheta \in \RR^d} \left\{f_j(\btheta) + \lambda_j \norm{\btheta -
		\widehat{\bbeta}}_2\right\}$, $\widehat{\btheta}_j \in B(\btheta^*, R)$, and Lemma \ref{lem:tech3}. We have
	completed the proof.
\end{proof}

\begin{lemma}[Robustness]\label{lem:vanillarobust}
	Let $\left\{f_j\right\}_{j=1}^{m}$ be convex. Suppose
	there exists $S \subseteq [m]$, $\btheta^* \in \RR^d$ and $0 < M < \infty$ such that for all $j \in
	S$, there exists a twice differentiable convex function $F_j$ such that
	$$
	\rho_j \bI \preceq \nabla^2 F_j(\btheta) \preceq L_j \bI, \;\;\;\;\;\; \forall \btheta \in
	B(\btheta^*, M)
	$$
	with some $0 < \rho_j \le L_j < +\infty$. Define
	$$
	f_S = \sum_{j \in S} f_j \square (\lambda_j \left\|\cdot\right\|_2), \;\;\;\;\;\;g_S =
	\sum_{j \in S} f_j,\;\;\;\;\;\; G_S = \sum_{j
		\in S} F_j,
	$$
	and denote
	$$
	\rho_S = \sum_{j \in S} \rho_j, \;\;\;\;\;\;
	\widehat{\btheta}_S \in \mathop{\mathrm{argmin}}_{\btheta \in \RR^d} f_S(\btheta), \;\;\;\;\;\; \lambda_{S^c} = \sum_{j \in
		S^c} \lambda_j.
	$$
	If, for some $0 \le \zeta_j < + \infty$, $j \in S$, and some $0 \le \zeta_S < + \infty$,  
	$$
	\sup_{\btheta \in B(\btheta^*, M)} \sup_{\bg_S \in \partial g_S(\btheta)} \norm{\bg_S -
		\nabla G_S(\btheta)}_2 \le \zeta_S,
	$$
	$$
	\sup_{\btheta \in B(\btheta^*, M)} \sup_{\bbf_j \in \partial f_j (\btheta)}\norm{\bbf_j -
		\nabla F_j(\btheta)}_2 \le \zeta_j, \;\;\;\;\;\; \forall j \in S,
	$$
	and
	$$
	\norm{\nabla F_j(\btheta^*)}_2 + \zeta_j + \frac{2 L_j}{\rho_S} \left(\bignorm{\sum_{k \in S} \nabla F_k(\btheta^*)}_2 +
	\zeta_S \right) + \frac{L_j\lambda_{S^c}}{\rho_S} < \lambda_j < \norm{\nabla F_j(\btheta^*)}_2 + \zeta_j + L_j M,
	$$
	for all $j \in S$, then 
	$$
	\norm{\widehat{\btheta}_S - \btheta^*}_2 \le \frac{\norm{\sum_{k \in S} \nabla F_k
			(\btheta^*)}_2}{\rho_S} + \frac{\zeta_S}{\rho_S},
	$$
	$\widehat{\btheta}_j = \widehat{\bbeta}$ for $j \in S$, and 
	$$
	\norm{\widehat{\bbeta} - \widehat{\btheta}_S}_2 \le \frac{\lambda_{S^c}}{\rho_S} +
	\frac{\zeta_S}{\rho_S}.
	$$
\end{lemma}

\begin{proof}[Proof of Lemma \ref{lem:vanillarobust}]
	Define $R = \min_{j \in S} \left\{(\lambda_j - \norm{\nabla F_j(\btheta^*)}_2 - \zeta_j) /
	L_j\right\}$. By Lemma \ref{lem:vanilladp} and its proof, we have 
	\begin{align*}
	&f_S(\btheta) = g_S(\btheta), \;\;\;\;\;\; \forall \btheta \in B(\btheta^*, R), \\
	&\widehat{\btheta}_S = \mathop{\mathrm{argmin}}_{\btheta \in \RR^d} f_S(\btheta) =
	\mathop{\mathrm{argmin}}_{\btheta \in \RR^d} g_S(\btheta) \subseteq B\left(\btheta^*,
	\frac{\norm{\sum_{k \in S} \nabla F_k (\btheta^*)}_2}{\rho_S} +
	\frac{\zeta_S}{\rho_S}\right),
	\end{align*}
	and $\nabla^2 G_S(\btheta) \succeq \rho_S \bI$ for all $\btheta \in B(\btheta^*, R)$. Define $f_0 =
	\sum_{j=1}^{m} f_j$. By
	Lemma \ref{lem:tech4} we have that 
	$$
	f_0 - f_S = \sum_{j \in S^c} f_j \square (\lambda_j\norm{\cdot}_2)
	$$
	is convex and $\lambda_{S^c}$-Lipschitz. Note that 
	$$
	R > \frac{1}{\rho_S} \left(2\bignorm{\sum_{k \in S} \nabla F_k(\btheta^*)}_2 + 2 \zeta_S +
	\lambda_{S^c}\right),
	$$
	and thus
	$$
	\frac{\lambda_{S^c}}{\rho_S} < R - \frac{2\norm{\sum_{k \in S} \nabla
			F_k(\btheta^*)}_2}{\rho_S} - \frac{\zeta_S}{\rho_S}.
	$$
	Denote the right-hand side above as $R_{S^c}$. Since $\lambda_{S^c} < \rho_S R_{S^c} - \zeta_S$ and $G_S$ is strongly convex in $B(\btheta^*, R_{S^c}) \subseteq B(\btheta^*, R)$, we can
	control the effect of $f_0 - f_S$ by Lemma \ref{lem:tech2}:
	$$
	\norm{\widehat{\bbeta} - \widehat{\btheta}_S}_2 \le \frac{\lambda_{S^c}}{\rho_S} +
	\frac{\zeta_S}{\rho_S}.
	$$
	Finally, note that $\widehat{\btheta}_j \in
	\mathop{\mathrm{argmin}}_{\btheta \in \RR^d} \left\{f_j(\btheta) + \lambda_j \norm{\btheta -
		\widehat{\bbeta}}_2\right\}$ for all $j \in S$. Since, for all $j \in S$, $\lambda_j > \norm{\nabla
		F_j(\btheta^*)}_2 + \zeta_j$ and
	\begin{align*}
	\norm{\widehat{\bbeta} - \btheta^*}_2 &\le \norm{\widehat{\btheta}_S - \btheta^*}_2 +
	\norm{\widehat{\bbeta} - \widehat{\btheta}_S}_2 \\ &\le \frac{\norm{\sum_{k \in S} \nabla 
			F_k (\btheta^*)}_2}{\rho_S}  + \frac{2 \zeta_S}{\rho_S} + \frac{\lambda_{S^c}}
	{\rho_S} \\ &\le \frac{\lambda_j - \norm{\nabla F_j(\btheta^*)}_2 - \zeta_j}{L_j} < M, 
	\end{align*}
	$\widehat{\btheta}_j = \widehat{\bbeta}$ for all $j \in S$ by Lemma \ref{lem:tech3}. 
\end{proof}

\begin{lemma}\label{lem:vanillaprep}
	Let $\left\{f_j\right\}_{j=1}^{m}$ be convex. Suppose
	there exists $S \subseteq [m]$, $\btheta^* \in \RR^d$, $0 < M < \infty$ and $0 < \rho \le L < \infty$ such that for all $j \in
	S$, there exists a twice differentiable convex function $F_j$ such that
	$$
	\rho \bI \preceq \nabla^2 F_j(\btheta) \preceq L \bI, \;\;\;\;\;\; \forall \btheta \in
	B(\btheta^*, M).
	$$
	Define $g_S = \sum_{j \in S} f_j$, $G_S = \sum_{j \in S} F_j$,
	and denote
	$$
	\widetilde{\btheta}_S \in
	\mathop{\mathrm{argmin}}_{\btheta \in \RR^d} g_S(\btheta),  \;\;\;\; \eta = \max_{j \in S} \left\{
	\norm{\nabla F_j (\btheta^*)}_2 + \zeta_j\right\},\;\;\;\;\;\; \kappa = \frac{L}{\rho}, \;\;\;\;\;\;
	\varepsilon = \frac{\left|S^c\right|}{\left|S\right|}. 
	$$
	Further suppose that, for some $0 \le \zeta_j < + \infty$, $j \in S$, and some $0 \le \zeta_S \le \sum_{j \in S} \zeta_j$,
	$$
	\sup_{\btheta \in B(\btheta^*, M)} \sup_{\bg_S \in \partial g_S(\btheta)} \norm{\bg_S -
		\nabla G_S(\btheta)}_2 \le \zeta_S,
	$$
	$$
	\sup_{\btheta \in B(\btheta^*, M)} \sup_{\bbf_j \in \partial f_j (\btheta)}\norm{\bbf_j -
		\nabla F_j(\btheta)}_2 \le \zeta_j, \;\;\;\;\;\; \forall j \in S.
	$$
	Take $\lambda_j = \lambda$ for all $j \in [m]$ and
	some $\lambda > 0$. If $\kappa\varepsilon < 1$ and
	$$
	\frac{3 \kappa \eta}{1 - \kappa \varepsilon} < \lambda < LM, 
	$$
	then $\widehat{\btheta}_j = \widehat{\bbeta}$ for $j \in S$, and 
	$$
	\norm{\widehat{\bbeta} - \widetilde{\btheta}_S}_2 \le \frac{\sum_{j \in S^c} \lambda_j +
		\zeta_S}{\rho \left|S\right|} \le \frac{\varepsilon \lambda}{\rho} + \frac{\zeta_S}{\rho \left|S\right|}, 
	$$
	$$
	\norm{\widetilde{\btheta}_S - \btheta^*}_2 \le \frac{\norm{\sum_{j \in S} \nabla F_j
			(\btheta^*)}_2}{\rho \left|S\right|} + \frac{\zeta_S}{\rho \left|S\right|}.
	$$
\end{lemma}

\begin{proof}[Proof of Lemma \ref{lem:vanillaprep}]
	From the assumption $\lambda > \frac{3 \kappa \eta}{1 - \kappa \varepsilon}$ we get
	$\lambda > 3 \kappa \eta + \kappa \varepsilon \lambda$ and for all $j \in S$, 
	\begin{align*}
	&\;\;\;\;\;\norm{\nabla F_j(\btheta^*)}_2 + \zeta_j + \frac{2 L}{\rho \left|S\right|}
	\left(\bignorm{\sum_{k \in S} \nabla F_k(\btheta^*)}_2 + \zeta_S \right) +
	\frac{L \lambda \left|S^c\right|}{\rho \left|S\right|} \\ & \le
	\eta + 2 \kappa \eta + \kappa \varepsilon \lambda \le 3 \kappa \eta + \kappa \varepsilon \lambda
	< \lambda = \lambda_j.
	\end{align*}
	Note that $\lambda_j = \lambda \le LM$ for all $j \in S$. The proof is finished by Lemma
	\ref{lem:vanillarobust} and its proof. Note that $\mathop{\mathrm{argmin}}_{\btheta \in \RR^d}
	\sum_{j \in S} f_j \square \left(\lambda_j \norm{\cdot}_2\right)(\btheta) =
	\mathop{\mathrm{argmin}}_{\btheta \in \RR^d} \sum_{j \in S} f_j(\btheta)$. 
\end{proof}

\section{PROOF OF SECTION \ref{sec:theory}} \label{sec:prooftheory}

\subsection{Proof of Theorem \ref{thm:statperson}}
The results follow immediately from Theorem \ref{thm:detperson} and Corollary \ref{cor:zetajmax}.

\subsection{Proof of Theorem \ref{thm:statvanilla}}
Take populations risks as $\{F_j\}_{j=1}^m$. By assumptions, $\rho, L, M \asymp 1$ and $\frac{\varepsilon}{1 - \varepsilon}
\lesssim 1$. Then, by Theorem \ref{thm:detvanilla}, there exist positive constants $\{C_i\}_{i=0}^2$
such that when  
$$
C_1 \max_{j \in S} \zeta_j < \lambda < C_2,
$$
we have for all $j \in S$,  
$$
\norm{\widehat{\btheta}_j - \btheta^*_j}_2 \le C_0 \left(\frac{\zeta_S}{m} + \min
\left\{\delta, \lambda\right\} + \varepsilon \lambda \right).
$$
Thus, we are to determine the order of $\max_{j \in S} \{\zeta_j\}$ and $\zeta_S$. 
Let $\bl: \RR^d \times \cX \rightarrow \RR^d$ be such that for every $\btheta \in
\RR^d$, $\bl(\btheta, \bxi) \in \partial_{\btheta} \ell(\btheta, \bxi)$. For any
$\{\cD_j\}_{j=1}^m$, denote $\bbf_j(\btheta) =
\frac{1}{n}\sum_{i=1}^{n} \bl(\btheta, \bxi_{ji})$ for all $j \in [m]$, and $\bF(\btheta) =
\sum_{j=1}^{m} \bbf_j(\btheta)$. By Corollary \ref{cor:zetajmax}, for some positive constants $c_1$
and $c_2$, the following holds with probability at least $1 - c_1 n^{-d}$: 
$$
\sup_{\btheta \in B(\btheta^*, r)} \norm{\bbf_j(\btheta) - \mathbb{E} \bbf_j(\btheta)
}_2 \le c_2 \sigma \sqrt{\frac{d \log n + \log m}{n}}, \;\;\;\;\;\; \forall j \in S. 
$$
By Corollary \ref{cor:zetasorder}, for some positive constants $c_3$ and $c_4$, the following holds with probability at
least $1 - c_3(m n)^{-d}$: 
$$
\sup_{\btheta \in B(\btheta^*, r)} \norm{\bF(\btheta) - \mathbb{E} \bF(\btheta)
}_2 \le c_4 \sigma \sqrt{\frac{dm \log m n}{n}}.
$$
When $1 \le \left(n^d\right)^{m - 1}m^{m - d}$, we have $\zeta_S \le m \cdot \max_{j \in S}
\zeta_j$. Theorem \ref{thm:statperson} applied to the tasks in $S^c$ yields
$$
\norm{\widehat{\btheta}_j - \btheta^*_j}_2 \le C_0 \lambda, \;\;\;\;\;\; \forall j \in S^c.
$$
The relation between $\widehat{\btheta}_j$ and $\mathop{\mathrm{argmin}}_{\btheta \in
	\RR^d} \{\sum_{j=1}^{m} f_j(\btheta)\}$ can be derived from Lemma \ref{thm:detvanilla}. We finish
the proof by taking union bounds and redefining the constants. 

\subsection{Supporting Lemmas for Section \ref{sec:theory}}
\begin{lemma}[Uniform First-Order Condition] \label{lem:ufoc}
	Define $R(\btheta) = \sum_{k=1}^{K} \left[\bl(\btheta, \bxi_k) - \mathbb{E} \bl(\btheta,
	\bxi_k)\right]$, where $\{\bxi_k\}_{k=1}^K$ are independent and $\bl$ satisfies Assumptions
	\ref{assump:conc}, \ref{assump:reg} and \ref{assump:smooth}. Choose some
	constant $0 < r < \infty$. Then, there exist constants $c_1, \, c_2 > 0$ such that with probability
	at least $1 - c_1 K^{-d}$, 
	\begin{equation} \label{eq:ufoc}
	\sup_{\btheta \in B(\btheta_0, r)} \norm{R(\btheta)}_2 \le c_2 \sigma \sqrt{d K
		\log K},\;\;\;\; \forall \btheta_0 \in \RR^d.
	\end{equation}
\end{lemma}

\begin{proof}[Proof of Lemma \ref{lem:ufoc}]
	By assumption, $\bl(\btheta, \bxi_k)$ is subgaussian for all $\btheta$ and $k$ and
	$\norm{\bl(\btheta, \bxi_k)}_{\psi_2} \le \sigma$. Thus, $R(\btheta)$
	is the sum of $K$ independent centered subgaussian random vectors and $\norm{R(\btheta)}_{\psi_2}
	\lesssim \sigma \sqrt{K}$. By Theorem 2.1 of \cite{HKZ12}, for some $c > 0$,   
	$$
	\mathbb{P} \left[ \norm{R(\btheta)}_2^2 > c^2K\sigma^2 \left(d + 2 \sqrt{dt} + 2t \right)
	\right] \le e^{-t}, \;\;\;\; \forall t \ge 0. 
	$$
	Since $2 \sqrt{dt} \le d + t$, 
	$$
	\mathbb{P} \left[ \norm{R(\btheta)}_2 > c \sigma \sqrt{K(d+t)}\, \right] \le e^{-t}, \;\;\;\; \forall t \ge 0. 
	$$
	Similar to the proof of Lemma 5.2 in \cite{V12}, $\forall \varepsilon > 0$, an $\varepsilon$-net
	$N_\varepsilon$ over $B(\btheta_0, r)$ satisfies
	$$
	\left|N_\varepsilon\right| \le \left(1 + \frac{2r}{\varepsilon}\right)^d.
	$$
	By union bounds,
	$$
	\mathbb{P} \left[\max_{\btheta \in N_\varepsilon} \norm{R(\btheta)}_2 \le c \sigma
	\sqrt{K(d+t)}\, \right] \ge 1 - \left(1 + \frac{2r}{\varepsilon}\right)^d e^{-t}, \;\;\;\; \forall t \ge 0. 
	$$
	Let $t = d \log K$, $\varepsilon = r\sqrt{K \log K}$. We have
	$$
	\left(1 + \frac{2r}{\varepsilon}\right)^d e^{-t} = \left(1 + \frac{2}{\sqrt{K \log K}}\right)^d K
	^{-d} \le cK^{-d}.
	$$
	Thus,
	$$
	\mathbb{P} \left[\max_{\btheta \in N_\varepsilon} \norm{R(\btheta)}_2 \le c_2\sigma
	\sqrt{d K \log K}\, \right] \ge 1 - c_1 K^{-d}. 
	$$
	By the proof of Proposition 3.4 in \cite{CLZ21}, with probability at least $1 - c_1 K^{-d}$, 
	$$
	\norm{R(\btheta_1) - R(\btheta_2)}_2 \le c_3 \sqrt{d K \log K},
	$$
	for any $\btheta_1 \in B(\btheta_0, r)$, $\btheta_2 \in N_\varepsilon$ such that $\norm{\btheta_1 -
		\btheta_2}_2 \le \varepsilon$. Taking union bounds over the two events, we have, with probability at
	least $1 - c_1 K^{-d}$,
	$$
	\sup_{\btheta \in B(\btheta_0, r)} \norm{R(\btheta)}_2 \le c_2 \sigma \sqrt{d K \log K}
	$$
	for some $c_1, c_2 > 0$. 
\end{proof}

\begin{corollary}[Maximum of $\{\zeta_j\}_{j \in S}$] \label{cor:zetajmax}
	Choose some constant $0 < r < \infty$. There exist positive constants $c_1$ and $c_2$ such that,
	with probability at least $1 - c_1	n^{-d}$, 
	$$
	\sup_{\btheta \in B(\btheta^*, r)} \norm{\bbf_j(\btheta) - \mathbb{E} \bbf_j(\btheta)
	}_2 \le c_2 \sigma \sqrt{\frac{d \log n  + \log m}{n}}, \;\;\;\;\;\; \forall j \in S. 
	$$
\end{corollary}

\begin{proof}[Proof of Corollary \ref{cor:zetajmax}]
	The proof is almost identical to the proof for Lemma \ref{lem:ufoc}. We can set $K = n$, $t = d
	\log n + \log \left|S\right|$ for all tasks in $S$. Taking union bounds and dividing by $n$ on
	both sides yield the result.  
\end{proof}

\begin{corollary}[Order of $\zeta_S$] \label{cor:zetasorder}
	Choose some constant $0 < r < \infty$. There exist positive constants $c_1$ and $c_2$ such that,
	with probability at least $1 - c_1	(mn)^{-d}$, 
	$$
	\sup_{\btheta \in B(\btheta^*, r)} \norm{\bF(\btheta) - \mathbb{E} \bF(\btheta)
	}_2 \le c_2 \sigma \sqrt{\frac{d m \log mn}{n}}. 
	$$
\end{corollary}

\begin{proof}[Proof of Corollary \ref{cor:zetasorder}]
	The proof is almost identical to the proof for Lemma \ref{lem:ufoc}. Setting $K =
	\left|S\right|n$ and dividing by $n$ on both sides yield the result. 
\end{proof}

\section{TECHNICAL LEMMAS}
\begin{lemma}\label{lem:tech1}
	Let $F, f: \RR^d \rightarrow \RR$ be convex. Denote $\bx^* \in \mathop{\mathrm{argmin}}_{\bx
		\in \RR^d} F(x)$ and $\widetilde{\bx} \in$ $\mathop{\mathrm{argmin}}_{\bx
		\in \RR^d} f(x)$. Suppose there exist $\bx_0 \in \RR^d$, $\bG_0 \in
	\partial F(\bx_0)$, $0 < r, \rho < \infty$ and $\zeta \ge 0$ such that $\norm{\bG_0}_2 \le r
	\rho / 2 - \zeta$,
	$$
	F(\bx) - F(\bx_0) \ge \inner{\bG_0}{\bx - \bx_0} + \frac{\rho}{2} \norm{\bx - \bx_0}_2^2,
	\;\;\;\;\;\;\forall \bx \in B(\bx_0, r),
	$$
	and
	$$
	\sup_{\bx \in B(\bx_0, r)} \left\{ \sup_{\bg \in \partial f(\bx), \bG \in \partial F(\bx)} \norm{\bg -
		\bG}_2\right\} \le \zeta.
	$$
	Then,
	$$
	\norm{\bx_0 - \bx^*}_2 \le \frac{2\norm{\bG_0}_2}{\rho} \;\;\;\;\;\; \text{and}
	\;\;\;\;\;\; \norm{\bx_0 - \widetilde{\bx}}_2 \le \frac{2\norm{\bG_0}_2}{\rho} + \frac{2\zeta}{\rho}. 
	$$
	Furthermore, if $\nabla^2 F(\bx) \succeq \rho \bI$ for all $\bx \in B(\bx_0,
	r)$, then $\bx^*$ is unique, and we have
	$$
	\norm{\bx_0 - \bx^*}_2 \le \frac{\norm{\nabla F(\bx_0)}_2}{\rho} \;\;\;\;\;\; \text{and}
	\;\;\;\;\;\; \norm{\bx_0 - \widetilde{\bx}}_2 \le \frac{\norm{\nabla F(\bx_0)}_2}{\rho} + \frac{\zeta}{\rho}. 
	$$
\end{lemma}

\begin{proof}[Proof of Lemma \ref{lem:tech1}]
	Let $G(\bx)$, $g(\bx)$ be subgradients of $F(\bx)$, $f(x)$, respectively, and define $\bx_t = (1-t)\bx_0 + t\bx$. We have
	$$
	f(\bx) - f(\bx_0) = \inner{\bx - \bx_0}{\int_{0}^{1} g(\bx_t) dt} 
	$$
	$$
	F(\bx) - F(\bx_0) = \inner{\bx - \bx_0}{\int_{0}^{1} G(\bx_t) dt} 
	$$
	This yields
	\begin{align*}
	f(\bx) - f(\bx_0) &\ge F(\bx) - F(\bx_0) - \norm{\bx -
		\bx_0}_2 \bignorm{\int_{0}^{1} \left[ g(\bx_t) - G(\bx_t)\right] dt}_2 \\
	&\ge \frac{\rho}{2} \norm{\bx - \bx_0}_2^2 - \norm{\bG_0}_2 \norm{\bx - \bx_0}_2 - \norm{\bx -
		\bx_0}_2 \int_{0}^{1} \norm{g(\bx_t) - G(\bx_t)}_2 dt.
	\end{align*}
 When $\| \bx - \bx_0 \|_2 < r$,  
	$$
	\int_{0}^{1} \norm{g(\bx_t) - G(\bx_t)}_2 dt \le \sup_{\bx \in B(\bx_0, r)}
	\left\{ \sup_{\bg \in \partial f(\bx), \bG \in \partial F(\bx)} \norm{\bg - \bG}_2\right\} \le \zeta,
	$$
	we have
	\begin{align*}
	f(\bx) - f(\bx_0) &\ge \frac{\rho}{2} \norm{\bx - \bx_0}_2^2 - \left(\norm{\bG_0}_2 + \zeta
	\right)\norm{\bx - \bx_0}_2 \\ &\ge \frac{\rho}{2} \norm{\bx - \bx_0}_2 \left( \norm{\bx -
		\bx_0}_2 - \frac{2\norm{\bG_0}_2}{\rho} - \frac{2 \zeta}{\rho} \right), \;\;\;\; \forall \bx \in B(\bx_0, r).
	\end{align*}
	Hence $f(\bx) - f(\bx_0) > 0$ when $2 \norm{\bG_0}_2 / \rho + 2 \zeta / \rho < \norm{\bx
		- \bx_0}_2 \le r$. When $\norm{\bx - \bx_0}_2 > r$, there exists $\bz = (1-t)\bx_0 + t\bx$ for some
	$t \in (0,1)$ such that $\norm{\bz - \bx_0}_2 = r$. By $f(\bz) > f(\bx_0)$ and the convexity of $f$,
	we have
	$$
	f(\bx_0) < f(\bz) \le (1-t) f(\bx_0) + t f(\bx)
	$$
	and thus $f(\bx) > f(\bx_0)$. Therefore, $\mathop{\mathrm{argmin}}_{\bx \in \RR^d} f(\bx) \subseteq
	B(\bx_0, 2 \norm{\bG_0}_2 / \rho + 2 \zeta / \rho) $. By a similar argument, $\mathop{\mathrm{argmin}}_{\bx \in \RR^d} F(\bx)$
	$\subseteq B(\bx_0, 2 \norm{\bG_0}_2 / \rho)$.
	
	Now, suppose that $\nabla^2 F(\bx) \succeq \rho \bI$ for all $\bx \in B(\bx_0,
	r)$. From $\mathop{\mathrm{argmin}}_{\bx \in \RR^d} F(\bx) \subseteq$ $B(\bx_0, 2 \norm{\bG_0}_2 /
	\rho) \subseteq B(\bx_0, r)$ and the strong convexity of $F$ therein we get the uniqueness of
	$\bx^*$. Then, $\nabla F(\bx^*) = 0$, $\norm{\bx^* - \bx_0}_2 \le r$, and from 
	\begin{align*}
	\norm{\nabla F(\bx_0)}_2 &= \norm{\nabla F(\bx_0) - \nabla F(\bx^*)}_2 \\ &= \bignorm{\left(\int_{0}^{1} \nabla^2
		F[(1-t)\bx^* + t\bx_0]dt\right)(\bx_0 - \bx^*)}_2 \ge \rho \norm{\bx_0 - \bx^*}_2
	\end{align*}
	we have $\norm{\bx_0 - \bx^*}_2 \le \norm{\nabla F(\bx_0)}_2 / \rho$. 
	
	Finally, by the definition of $\widetilde{\bx}$, $\bm{0} \in \partial f(\widetilde{\bx})$. We have
	\begin{align*}
	\norm{\nabla F(\bx_0)}_2 &\ge \norm{\nabla F(\bx_0) - \nabla F(\widetilde{\bx})}_2 - \norm{\nabla
		F(\widetilde{\bx}) - \bm{0}}_2 \\ &\ge \bignorm{\left(\int_{0}^{1} \nabla^2
		F[(1-t)\widetilde{\bx} + t\bx_0]dt\right)(\bx_0 - \widetilde{\bx})}_2 - \sup_{\bg \in \partial
		f(\widetilde{\bx})} \norm{\bg - \nabla F(\widetilde{\bx})}_2 \\ &\ge \rho \norm{\bx_0 - \widetilde{\bx}}_2
	- \zeta,
	\end{align*}
	from which we have $\norm{\bx_0 - \widetilde{\bx}}_2 \le \norm{\nabla F(\bx_0)}_2 / \rho + \zeta /
	\rho$. We have completed the proof. 
\end{proof}

\begin{lemma}\label{lem:tech2}
	Let $F, f: \RR^d \rightarrow \RR$ be convex functions and $\bx^* = \mathop{\mathrm{argmin}}_{\bx \in \RR^d} F(\bx)$. Suppose $F$ is differentiable and 
	$$
	\rho \bI \preceq \nabla^2 F(\bx), \;\;\;\;\;\; \forall \bx \in B(\bx^*, r)
	$$
	holds for some $0 < \rho < \infty$ and $0 < r < \infty$. If, for some $\zeta
	\ge 0$,
	$$
	\sup_{\bx \in B(\bx^*, r)} \sup_{\bm{f} \in \partial f(\bx)} \norm{\bm{f} - \nabla F(\bx)}_2 
	\le \zeta,
	$$
	then 
	\begin{equation}\label{eq:derbound}
	\norm{\bm{f}}_2 \ge \rho \min\left\{ \norm{\bx - \bx^*}_2, r \right\} - \zeta,
	\;\;\;\;\;\; \forall \bm{f} \in \partial f(\bx).
	\end{equation}
	If $g: \RR^d \rightarrow \RR$ is convex and $\lambda$-Lipschitz for some $0 \le \lambda < \rho r -
	\zeta$, then all minimizers of $f(\bx) + g(\bx)$ belong to $B(\bx^*, \lambda / \rho + \zeta / \rho)$.
\end{lemma}

\begin{proof}[Proof of Lemma \ref{lem:tech2}]
	The optimality of $\bx^*$ and the strong convexity of $F$ near $\bx^*$ implies $\nabla F(\bx^*)
	= \bm{0}$. Choose any $\bm{f} \in \partial f(\bx)$. If $0 < \norm{\bx - \bx^*}_2 \le r$, then
	\begin{align*}
	&\;\;\;\;\; \norm{\bbf}_2\norm{\bx - \bx^*}_2 \ge \inner{\bbf}{\bx - \bx^*}\\ &= \inner{\bbf - \nabla F(\bx)}{\bx -
		\bx^*} + \inner{\nabla F(\bx) - \nabla F(\bx^*)}{\bx - \bx^*}\\ &\ge -\sup_{\bbf \in \partial
		f(\bx)} \norm{\bbf - \nabla F(\bx)}_2 \norm{\bx-\bx^*}_2 + \inner{\left(\int_{0}^{1} \nabla^2
		F[(1-t)\bx^* + t\bx] dt \right)(\bx - \bx^*)}{\bx - \bx^*} \\ & \ge \rho \norm{\bx - \bx^*}_2^2 -
	\zeta \norm{\bx - \bx^*}_2
	\end{align*}
	and $\norm{\bbf}_2 \ge \rho \norm{\bx - \bx^*}_2 - \zeta$. If $\norm{\bx - \bx^*}_2 > r$, there exists
	$\bz = (1-t)\bx^* + t\bx$ for some $t\in(0,1)$ such that $\norm{\bz-\bx^*}_2 = r$. Choose any $\bbf'
	\in \partial f(\bz)$. By the convexity of $f$, $\inner{\bbf - \bbf'}{\bx - \bz} \ge 0$ and hence
	$\inner{\bbf - \bbf'}{\bx - \bx^*} \ge 0$. Then,
	\begin{align*}
	&\;\;\;\;\; \norm{\bbf}_2\norm{\bx - \bx^*}_2 \ge \inner{\bbf}{\bx - \bx^*}\\ &=  \inner{\bbf -
		\bbf'}{\bx - \bx^*} + \inner{\bbf' -
		\nabla F(\bz)}{\bx -
		\bx^*} + \inner{\nabla F(\bz) - \nabla F(\bx^*)}{\bx - \bx^*}\\ &\ge -\sup_{\bbf' \in \partial
		f(\bz)} \norm{\bbf' - \nabla F(\bz)}_2 \norm{\bx-\bx^*}_2 + \inner{\left(\int_{0}^{1} \nabla^2
		F[(1-t)\bx^* + t\bz] dt \right)(\bz - \bx^*)}{\bx - \bx^*} \\ & \ge \rho r\norm{\bx - \bx^*}_2 -
	\zeta \norm{\bx - \bx^*}_2
	\end{align*}
	and $\norm{\bbf}_2 \ge \rho r - \zeta$. We have verified (\ref{eq:derbound}).
	
	Choose any $\widehat{\bx} \in \mathop{\mathrm{argmin}}_{\bx \in \RR^d} \left\{ f(\bx) + g(\bx)
	\right\}$. There exist $\bm{f} \in \partial f(\widehat{\bx})$ and $\bm{g} \in \partial g(\widehat{\bx})$ such that
	$\bm{f} + \bm{g} = 0$. The Lipschitz property of $g$ yields $\norm{\bm{f}}_2 = \norm{\bm{g}}_2 \le
	\lambda$. Since $\lambda < \rho r - \zeta$, we obtain from (\ref{eq:derbound}) that 
	$$
	\rho r - \zeta > \lambda \ge \norm{\bm{f}}_2 \ge \rho \min \left\{\norm{\widehat{\bx} - \bx^*}_2,
	r\right\} - \zeta,
	$$
	which leads to $\norm{\widehat{\bx} - \bx^*}_2 \le \lambda / \rho + \zeta / \rho$. We have completed the proof.
\end{proof}

\begin{lemma}\label{lem:tech3}
	Let $F, f: \RR^d \rightarrow \RR$ be convex functions. Suppose $F$ is twice differentiable and 
	$$
	\nabla^2 F(\bx) \preceq L \bI, \;\;\;\;\;\; \forall \bx \in B(\bx^*, M)
	$$
	holds for some $\bx^* \in \RR$, $0 < L < \infty$ and $0 < M < \infty$. If, for some $\zeta \ge 0$
	and $\lambda > 0$,
	$$
	\sup_{\bx \in B(\bx^*, M)} \sup_{\bbf \in \partial f(\bx)}\norm{\bbf - \nabla F(\bx)}_2 
	\le \zeta,
	$$
	and $\lambda > \norm{\nabla F(\bx^*)}_2 + \zeta$, then
	$$
	f(\bx) = f\square(\lambda \norm{\cdot}_2) (\bx) \;\;\;\;\;\; \text{and} \;\;\;\;\;\;
	\mathop{\mathrm{argmin}}_{\bx' \in \RR^d} \left\{f(\bx') + \lambda \norm{\bx - \bx'}_2\right\}
	= \bx
	$$
	hold for all $\bx \in B(\bx^*, \min\left\{(\lambda - \norm{\nabla F(\bx^*)}_2 - \zeta) / L, M\right\})$.
\end{lemma}

\begin{proof}[Proof of Lemma \ref{lem:tech3}]
	For any $\bx$ such that $\norm{\bx - \bx^*}_2 \le \min\left\{(\lambda
	- \norm{\nabla F(\bx^*)}_2 - \zeta) / L, M\right\}$, we have 
	\begin{align*}
	\sup_{\bm{f} \in \partial f(\bx)} \norm{\bbf}_2 &\le \sup_{\bbf \in \partial f(\bx)} \norm{\bbf -
		\nabla F(\bx)}_2 + \norm{\nabla F(\bx) - \nabla F(\bx^*)}_2 + \norm{\nabla F(\bx^*)}_2 \\ &\le \zeta
	+ L \norm{\bx - \bx^*}_2 + \norm{\nabla F(\bx^*)}_2\le \lambda.
	\end{align*}
	Recall $f \square (\lambda \norm{\cdot}_2)(\bx)=\inf_{\bx' \in \RR} \left\{f(\bx') +
	\lambda \norm{\bx - \bx'}_2\right\}$. Define $h(\bx') = f(\bx') + \lambda \norm{\bx-\bx'}_2$. Since
	$\norm{\bbf}_2 \le \lambda$ for any $\bbf \in \partial f(\bx)$, it follows from $\partial
	\norm{\bx - \bx'}_2 \Big|_{\bx' = \bx} = \left\{\bg\in \RR^d : \norm{\bg}_2 \le 1\right\}$ that
	$\bm{0} \in \partial h(\bx)$. Thus, $\mathop{\mathrm{argmin}}_{\bx'\in \RR^d} \left\{f(\bx') +
	\lambda \norm{\bx - \bx'}_2\right\} = \bx$, and $f(\bx) = f \square (\lambda
	\norm{\cdot}_2)(\bx)$.
\end{proof}

\begin{lemma}\label{lem:tech4}
	If $f: \RR^d \rightarrow \RR$ is convex, $\inf_{\bx \in \RR^d} f(x) > -\infty$, $g: \RR^d
	\rightarrow \RR$ is convex and $L$-Lipschitz with respect to a norm $\norm{\cdot}$ for some $L \ge
	0$, then $f \square g$ is convex and $L$-Lipschitz with respect to $\norm{\cdot}$.
\end{lemma}
\begin{proof}[Proof of Lemma \ref{lem:tech4}]
	The lemma is directly taken from Lemma E.4 of \cite{DW22}. 
\end{proof}

\begin{lemma}\label{lem:tech5}
	Let $f:~\RR^d \to \RR$ be a convex function, $\bg:~\RR^d \to \RR^d$ be a continuous vector field, and $\Omega \subseteq \RR^d$ be an open set. Choose any $\bbf:~\RR^d \to \RR^d$ such that for every $\bx \in \RR^d$, $\bbf(\bx) \in \partial f(\bx)$. Then,
	\[
	\sup_{ \bx \in \Omega } \sup_{ \bh \in \partial f(\bx) } \| \bh - \bg (\bx) \|_2
	=
	\sup_{ \bx \in \Omega } \| \bbf (\bx) - \bg (\bx) \|_2.
	\]
\end{lemma}

\begin{proof}[Proof of Lemma \ref{lem:tech5}]
	Define $M = \sup_{ \bx \in \Omega } \| \bbf (\bx) - \bg (\bx) \|_2$. The claim is trivially true when $M = \infty$.
	Below we assume that $M < \infty$. 
	Choose an arbitrary $\bx \in \Omega$ and any $\bh \in \partial f(\bx)$. It suffices to prove that $\| \bh - \bg(\bx) \|_2 \leq M$. Let $\bv = [\bh - \bg(\bx)] / \| \bh - \bg(\bx) \|_2$ if $\bh \neq \bg(\bx)$; otherwise, let $\bv$ be any unit-norm vector. By construction, $\| \bh - \bg (\bx) \|_2 = 
	\langle   \bh - \bg(\bx) , \bv \rangle$.
	Define a univariate function $F(t) = f(\bx + t \bv )$, $t \in \RR$. It is convex and satisfies
	\begin{align*}
	\partial F(t) = \{ \langle \bu, \bv  \rangle :~ \bu \in \partial f( \bx + t \bv )  \} .
	\end{align*}
	In particular, we have $\langle  \bh, \bv \rangle \in \partial F(0)$ and $\langle  \bbf(\bx + n^{-1} \bv ), \bv \rangle \in \partial F(n^{-1} )$, $\forall n$. By the convexity of $F$, $\{ \langle  \bbf(\bx + n^{-1} \bv ), \bv \rangle \}_{n=1}^{\infty}$ is non-increasing and $\langle  \bh, \bv \rangle \leq \langle  \bbf(\bx + n^{-1} \bv ), \bv \rangle $, $\forall n$.
	Therefore,
	\begin{align*}
	\| \bh - \bg (\bx) \|_2 = 
	\langle   \bh - \bg(\bx) , \bv \rangle \leq \lim_{n\to \infty} \langle \bbf(\bx + n^{-1} \bv )   ,  \bv \rangle - \langle  \bg(\bx) ,  \bv \rangle. 
	\end{align*}
	The continuity of $\bg$ yields $\bg(\bx)  = \lim_{n\to\infty} g (\bx + n^{-1} \bv ) $ and
	\begin{align}
	\| \bh - \bg (\bx) \|_2  \leq \lim_{n\to \infty} \langle \bbf ( \bx + n^{-1} \bv ) - \bg (\bx + n^{-1} \bv )   ,  \bv \rangle.
	\label{ineq-lem-1}
	\end{align}
	
	Define $B(r) = \{ \by \in \RR^d :~ \| \by - \bx \|_2 \leq r \} $ for any $r \geq 0$. Since $\bx \in \Omega$ and $\Omega$ is open, there exists $\delta > 0$ such that $B(\delta) \subseteq \Omega$. For any $n > 1 / \delta$, we have $\bx + n^{-1} \bv \in B(\delta) \subseteq \Omega$ and thus
	\begin{align}
	\langle \bbf (\bx + n^{-1} \bv ) - \bg (\bx + n^{-1} \bv )   ,  \bv \rangle
	& \leq \| \bbf (\bx + n^{-1} \bv ) - \bg (\bx + n^{-1} \bv ) \|_2 \notag \\
	& \leq 
	\sup_{ \by \in \Omega } \| \bbf (\by) - \bg (\by) \|_2 = 
	M .
	\label{ineq-lem-2}
	\end{align}
	The inequalities \eqref{ineq-lem-1} and \eqref{ineq-lem-2} imply that $\| \bh - \bg (\bx) \|_2 \leq M$.
\end{proof}

{
\bibliographystyle{ims}
\bibliography{bib}
}

\end{document}